\documentclass[a4paper,UKenglish,cleveref, autoref, thm-restate]{lipics-v2021}
\usepackage{apxproof}
\usepackage{todonotes}

\title{A canonical generalization of OBDD} 

\author{Florent Capelli}{Univ. Artois, CNRS, UMR 8188, CRIL, F-62300 Lens, France}{capelli@cril.fr}{https://orcid.org/0000-0002-2842-8223}{}

\author{YooJung Choi}{Arizona State University}{yj.choi@asu.edu}{https://orcid.org/0009-0009-8102-4170}{}

\author{Stefan Mengel}{Univ. Artois, CNRS, UMR 8188, CRIL, F-62300 Lens, France}{mengel@cril-lab.fr}{https://orcid.org/0000-0003-1386-8784}{}

\author{Martín Muñoz}{Univ. Artois, CNRS, UMR 8188, CRIL, F-62300 Lens, France}{munoz@cril.fr}{https://orcid.org/0009-0003-3294-6159}{}

\author{Guy Van den Broeck}{University of California, Los Angeles}{guyvdb@cs.ucla.edu}{https://orcid.org/0000-0003-3434-2503}{}

\authorrunning{F. Capelli, Y. Choi, S. Mengel, M. Muñoz and G. Van den Broeck} 
\Copyright{Florent Capelli, YooJung Choi, Stefan Mengel, Martín Muñoz and Guy Van den Broeck}



\category{} 

\relatedversion{} 

\nolinenumbers 

\EventEditors{John Q. Open and Joan R. Access}
\EventNoEds{2}
\EventLongTitle{42nd Conference on Very Important Topics (CVIT 2016)}
\EventShortTitle{CVIT 2016}
\EventAcronym{CVIT}
\EventYear{2016}
\EventDate{December 24--27, 2016}
\EventLocation{Little Whinging, United Kingdom}
\EventLogo{}
\SeriesVolume{42}
\ArticleNo{23}

\usepackage{xifthen}

\usepackage{algorithm,algpseudocode}
\let\poly\relax
\DeclareMathOperator{\poly}{poly}
\DeclareMathOperator{\lit}{lit}

\NewDocumentCommand{\wrt}{}{with respect to}

\NewDocumentCommand{\N}{}{\ensuremath{\mathbb{N}}}

\NewDocumentCommand{\ua}{O{x} O{d}}{\ensuremath{\langle #1/#2 \rangle}}
\NewDocumentCommand{\restr}{m m}{\ensuremath{#1|_{#2}}}

\NewDocumentCommand{\calT}{}{\ensuremath{\mathcal{T}}}

\NewDocumentCommand{\calP}{}{\ensuremath{\mathcal{P}}}

\NewDocumentCommand{\sub}{}{\ensuremath{\mathsf{sub}}}

\NewDocumentCommand{\out}{O{C}}{\ensuremath{\mathsf{out}(#1)}}
\NewDocumentCommand{\sib}{O{g} O{g_1}}{\ensuremath{\mathsf{sib}(#2,#1)}}
\NewDocumentCommand{\HWB}{}{\ensuremath{\mathsf{HWB}}}
\NewDocumentCommand{\MUX}{}{\ensuremath{\mathsf{MUX}}}

\NewDocumentCommand{\hwidth}{m m m}{%
  \ifthenelse{\isempty{#2}}%
  {\ensuremath{\mathsf{#3}(#1)}}%
  {\ensuremath{\mathsf{#3}(#1, #2)}}%
}
\NewDocumentCommand{\how}{m O{}}{\hwidth{#1}{#2}{how}}
\NewDocumentCommand{\fhow}{m O{}}{\hwidth{#1}{#2}{fhow}}
\NewDocumentCommand{\bhow}{m O{}}{\beta\text{-}\hwidth{#1}{#2}{how}}
\NewDocumentCommand{\bfhow}{m O{}}{\beta\text{-}\hwidth{#1}{#2}{fhow}}
\NewDocumentCommand{\sfhow}{m O{}}{\hwidth{#1}{#2}{sfhow}}
\RenewDocumentCommand{\show}{m O{}}{\hwidth{#1}{#2}{show}}
\NewDocumentCommand{\htw}{m O{}}{\hwidth{#1}{#2}{htw}}
\NewDocumentCommand{\fhtw}{m O{}}{\hwidth{#1}{#2}{fhtw}}
\NewDocumentCommand{\bfhtw}{m O{}}{\beta\text{-}\hwidth{#1}{#2}{fhtw}}
\NewDocumentCommand{\bhtw}{m O{}}{\beta\text{-}\hwidth{#1}{#2}{htw}}
\NewDocumentCommand{\bpfhtw}{m O{}}{\beta\text{-}\hwidth{#1}{#2}{fhtw'}}
\NewDocumentCommand{\bphtw}{m O{}}{\beta\text{-}\hwidth{#1}{#2}{htw'}}
\NewDocumentCommand{\nsw}{m O{}}{\hwidth{#1}{#2}{nsw}}
\NewDocumentCommand{\tw}{m O{}}{\hwidth{#1}{#2}{tw}}
\NewDocumentCommand{\ptw}{m O{}}{\hwidth{#1}{#2}{ptw}}
\NewDocumentCommand{\itw}{m O{}}{\hwidth{#1}{#2}{itw}}
\NewDocumentCommand{\fw}{m O{}}{\hwidth{#1}{#2}{fw}}

\NewDocumentCommand{\ccw}{m O{}}{\hwidth{#1}{#2}{ccw}}

\NewDocumentCommand{\gprim}{O{F}}{\ensuremath{{\mathsf{Prim}(#1)}}}
\NewDocumentCommand{\ginc}{O{F}}{\ensuremath{{\mathsf{Inc}(#1)}}}

\NewDocumentCommand{\syncdecdnnf}{O{\calP}}{\ensuremath{\mathsf{sync\text{-}decDNNF(#1)}}}

\usepackage{color}

\RenewDocumentCommand{\to}{}{\ensuremath{\rightarrow}}
\NewDocumentCommand{\var}{m}{\ensuremath{\mathsf{var}(#1)}}

\NewDocumentCommand{\BPO}{O{H} O{p}}{\ensuremath{\mathsf{BPO}(#1,#2)}}

\NewDocumentCommand{\prim}{O{F}}{\ensuremath{\mathsf{G_{prim}}(#1)}}
\NewDocumentCommand{\inc}{O{F}}{\ensuremath{\mathsf{G_{inc}}(#1)}}
\NewDocumentCommand{\neigh}{O{v} O{H}}{\ensuremath{N_{#2}(#1)}}
\NewDocumentCommand{\oneigh}{O{v} O{H}}{\ensuremath{N^*_{#2}(#1)}}

\NewDocumentCommand{\inputs}{O{v}}{\ensuremath{\mathsf{inputs}}(#1)}

\NewDocumentCommand{\rel}{O{C}}{\ensuremath{\mathsf{rel}({#1})}}
\NewDocumentCommand{\nrel}{O{g} O{d}}{\ensuremath{\mathsf{nrel}({#1,#2})}}

\NewDocumentCommand{\ine}{O{g}}{\ensuremath{\mathsf{in_E}}(#1)}
\NewDocumentCommand{\oute}{O{g}}{\ensuremath{\mathsf{out_E}}(#1)}
\NewDocumentCommand{\ing}{O{g}}{\ensuremath{\mathsf{in}}(#1)}
\NewDocumentCommand{\outg}{O{g}}{\ensuremath{\mathsf{out}}(#1)}
\NewDocumentCommand{\edges}{O{C} O{}}{\ensuremath{\mathsf{edges}_{#2}(#1)}}
\NewDocumentCommand{\gates}{O{C} O{}}{\ensuremath{\mathsf{gates}_{#2}(#1)}}

\DeclareMathOperator{\size}{size}

\RenewDocumentCommand{\size}{m}{\ensuremath{\|#1\|}}

\NewDocumentCommand{\ccq}{O{Q} O{\tau}}{\ensuremath{\mathsf{cc}({#1}, {#2})}}
\NewDocumentCommand{\pcc}{O{Q} O{i} O{\pi}}{\ensuremath{{#3}\text{-}\mathsf{cc}(#1,#2)}}
\NewDocumentCommand{\apcc}{O{Q} O{\pi}}{\ensuremath{{#2}\text{-}\mathsf{cc}(#1)}}

\NewDocumentCommand{\dpll}{O{Q} O{\tau} O{\pi}}{\ensuremath{\mathsf{DPLL}({#1}, {#2}, {#3})}}
\NewDocumentCommand{\rec}{O{Q} O{\pi}}{\ensuremath{\mathbf{R}_{#1,#2}}}
\NewDocumentCommand{\gatoms}{O{i} O{Q} O{\pi}}{\ensuremath{G_{#2,#3}^{#1}}}

\NewDocumentCommand{\bin}{m O{b}}{\ensuremath{\widetilde{#1}^{#2}}}
\NewDocumentCommand{\debin}{m O{b}}{\ensuremath{\bar{#1}^{#2}}}
\NewDocumentCommand{\ans}{O{Q}}{\ensuremath{{{\llbracket #1 \rrbracket}}}}
\NewDocumentCommand{\sans}{O{Q} O{D}}{\ensuremath{{{\llbracket #1 \rrbracket}_{#2}}}}

\NewDocumentCommand{\fN}{O{i} O{H} O{\pi}}{\ensuremath{\mathcal{N}_{#2,#3}(#1)}}

\NewDocumentCommand{\tup}{O{}}{\ensuremath{\langle #1 \rangle}}

\begin{document}

\maketitle

\begin{abstract}
  We introduce Tree Decision Diagrams (TDD) as a model for Boolean functions that generalizes OBDD. They can be seen as a restriction of structured d-DNNF; that is, d-DNNF that respect a vtree $T$. We show that TDDs
enjoy the same tractability properties as OBDD, such as model counting, enumeration, conditioning, and apply, and are more succinct. In particular, we show that CNF formulas of treewidth $k$ can be represented by TDDs of FPT size, which is known to be impossible for OBDD.  We study the complexity of compiling CNF formulas into deterministic TDDs via bottom-up compilation and relate the complexity of this approach with the notion of factor width introduced by Bova and Szeider.

\end{abstract}

\section{Introduction}

\emph{Knowledge compilation} is the systematic study of different representations of knowledge, often in the form of Boolean functions, but also for preferences~\cite{FargierMM24}, actions in planning~\cite{pliego2015decision}, product configuration~\cite{sundermann2024benefits,renault1}, databases~\cite{OlteanuZ12}, etc. 
To compare different data structures representing the same type of data, following the groundbreaking work of Darwiche and Marquis~\cite{DarwicheM2002}, one analyzes them with respect to a list of potential desirable properties that they might have, generally a set of tractable operations and queries on them. There is a general observed trade-off between usefulness (what can one do efficiently with a data structure?) and succinctness (how small is the representation in a specific form?): on one end of the spectrum, there are representations like OBDD that allow many useful operations, but are rather verbose. On the other end, there are, e.g.,~DNNF that are far more succinct but allow only few operations efficiently. Knowledge compilation explores the space between these two extremes and aims to provide representation languages with different trade-offs for different applications.

One important operation in knowledge compilation is the so-called {\em apply} operation which is, given two representations of the same format, to compute a representation of a target Boolean combination, most importantly, their conjunction. The most prominent knowledge compilation languages that support this operation efficiently are OBDD~\cite{Bryant92} 
and SDD~\cite{Darwiche11}. The apply operation is of special practical importance because it is often used as the basis of \emph{bottom-up} compilation of systems of constraints into a different target representation. The idea is to first compile the individual constraints into the target format and then iteratively conjoin them with the apply operation. In particular, this is the most common approach for constructing OBDD~\cite{cudd} and SDD~\cite{Choi_Darwiche_2013}. To avoid size blow-ups during bottom-up compilation, it is common to try to shrink the currently compiled form, which for OBDD is possible because they can be turned into a canonical minimal form, i.e., a form of minimal size and unique, up to isomorphism, among all equivalent OBDD with the same variable order. Canonicity is also useful to efficiently test equivalence between two given OBDD. For SDD, which are in general exponentially smaller than OBDD~\cite{Bova16}, the situation is more complicated~\cite{van2015role}: while they have a canonical form, it is not minimal---in fact, it can be exponentially larger than the smallest equivalent SDD. 
Moreover, canonical SDD are not stable under conjunction, 
as the conjunction of two canonical SDD can become exponentially larger than each after canonization.

In this paper, we introduce and analyze a new knowledge compilation language which we call Tree Decision Diagrams (TDD). We show that TDD have various desirable properties of OBDD, such as having an efficient apply algorithm, a canonical form that is also minimal, and an efficient algorithm to find it for any given non-minimal TDD. We show that, as is the case for OBDD, the size of a canonical TDD can be characterized by certain subfunction counts which gives a very clean understanding of which functions can efficiently be compiled into a TDD. 
We highlight that, in contrast to SDD, canonical TDD can be efficiently combined via apply into a new canonical TDD.



Since TDD have efficient apply and minimization algorithms, they are a good target language for bottom-up compilation.
As a proof of concept, we present a simple algorithm that allows compiling CNF formulas and circuits of bounded treewidth efficiently. While compilation results in this setting were known before~\cite{Darwiche04,PipatsrisawatD10,BovaCMS15,AmarilliCMS20,BovaS17}, our approach compiles into a more restricted language with better properties. We highlight that these results had rather involved dynamic programming solutions, in contrast to our compilation algorithm which simply performs apply and minimization in a bottom-up fashion.
Our results depend on the characterization by subfunction counts. However, crucially, this argument is only used in the analysis and not in the algorithm.


The paper is organized as follows: \cref{sec:prelim} introduces necessary preliminaries. \cref{sec:prop:dd} defines the notion of TDD, \cref{sec:transformations} presents the transformations that are tractable for TDDs. \cref{sec:prop:minimization} contains the minimization procedures for TDD and shows that they are canonical. \cref{sec:prop:struct-cnf} establishes bottom-up compilation of TDD and uses it to compile bounded treewidth formulas and circuits. Finally, \cref{sec:prop:tdd-vs-rest} compares TDDs with other representation languages. Due to page limit, we moved most of the proofs to the appendix. Statements whose proof can be found in this appendix are marked with a $(\star)$ symbol. 


\section{Preliminaries}\label{sec:prelim}

\textbf{Assignments and Boolean functions.} Given two sets $A$ and $B$, we denote by $B^A$ the set of functions from $A$ to $B$. When $B = \{0,1\}$, we will often write $2^A$ to denote the set of assignments from a set $A$ to $\{0,1\}$. An element $\tau \in 2^A$ is called a \emph{Boolean assignment over variables $A$}, and we will often just write ``assignment'' when it is clear from context that it is Boolean. A \emph{partial (Boolean) assignment over variables $X$} is an element of $2^Y$ for some $Y \subseteq X$. Given two assignments $\tau_1 \in 2^X$ and $\tau_2 \in 2^Y$ with $X \cap Y = \emptyset$, we denote by $\tau_1 \times \tau_2$ the assignment over variables $X \cup Y$ such that $(\tau_1 \times \tau_2)(z) = \tau_1(z)$ if $z \in X$ and $(\tau_1 \times \tau_2)(z) = \tau_2(z)$ if $z \in Y$. We denote by $\ua[x][0]$ (resp. $\ua[x][1]$) the assignment in $2^{\{x\}}$ mapping $x$ to $0$ (resp. to $1$). We will also use the notation $\tup[x_1 / b_1, \dots, x_k / b_k]$ to denote the assignment in $2^{\{x_1,\dots,x_k\}}$ mapping $x_i$ to $b_i$. Given $\tau \in 2^X$ and $Y \subseteq X$, we denote by $\restr{\tau}{Y}$ the assignment in $2^Y$ such that $\restr{\tau}{Y}(y) = \tau(y)$ for every $y \in Y$.

A \emph{Boolean function $f$ over variables $X$} is a mapping from $2^X$ to $\{0,1\}$. An assignment $\tau$ such that $f(\tau) = 1$ is said to \emph{satisfy} $f$ and is alternatively called a \emph{satisfying assignment} or a \emph{model}. Given a Boolean function $f$ over variables $X$ and $Y \subseteq X$, we denote by $\restr{f}{Y}$ the Boolean function over variables $Y$ whose models are $\{\restr{\tau}{Y} \mid f(\tau)=1\}$. We denote by $\neg f$ the negation of $f$ and by $f \wedge g$ (resp. $f \vee g$) the conjunction (resp. disjunction) of $f$ and $g$.

\textbf{Conjunctive Normal Form Formulas.} 
Given a set $X$ of variables, a literal over $X$ is either $x \in X$ or $\neg x$. We let $\lit(X)$ be the set of literals over $X$, and for $\ell \in \lit(X)$, we denote by $\var{\ell}$ its underlying variable (that is, $x = \var{x} = \var{\neg x}$). For an assignment $\tau \in 2^X$, we naturally extend it to literals by defining $\tau(\neg x) = 1-\tau(x)$.
A \emph{clause} $c$ is a set of literals, interpreted as their disjunction and written as $c = \ell_1 \vee \dots \vee \ell_k$;
we let $\var{c} = \{\var{\ell}\mid \ell \in c\}$. 
An assignment $\tau$ \emph{satisfies a clause $c$} if there exists $\ell \in c$ such that $\tau$ is defined on $\var{\ell}$ and $\tau(\ell) = 1$.
A \emph{Conjunctive Normal Form (CNF) formula} $F$ is a set of clauses, interpreted as their conjunction and often denoted $F=c_1 \wedge \dots \wedge c_m$. We let $\var{F} = \bigcup_{c \in F} \var{c}$. 
An assignment $\tau$ \emph{satisfies $F$}
if for every clause $c \in F$, $\tau$ satisfies $c$. The Boolean function defined by a CNF formula is the Boolean function over $\var{F}$ whose models are exactly the assignments over $\var{F}$ that satisfy $F$. We will often identify a CNF formula or a clause with the Boolean function it represents and use the notation we defined for Boolean functions directly on formulas. For example, we write $F \models c$ whenever every satisfying assignment of $F$ is also a satisfying assignment for $c$.
The \emph{size $\size{F}$} of $F$ is defined as $\sum_{c \in F} |\var{c}|$.

A CNF formula can be conditioned by a partial assignment: for $\tau \in 2^Y$, we let $F[\tau]$ be the CNF formula obtained as follows. We remove from $F$ every clause $c$ containing a literal $\ell$ such that $\tau(\ell)=1$. In the remaining clauses, we remove every literal $\ell$ such that $\tau(\ell)=0$. An assignment $\sigma \in 2^{\var{F} \setminus Y}$ satisfies $F[\tau]$ if and only if $\sigma \times \tau$ satisfies $F$.

\textbf{Graphs of CNF formulas.} We characterize the structure of CNF formulas using graphs. Given a CNF formula $F$ over variables $X$, the \emph{primal graph} of $F$, denoted by $\gprim = (X,E)$, is the graph whose vertices are the variables of $F$ and which has an edge $\{x,y\}$ if and only if there is a clause $c \in F$ such that $x,y \in \var{c}$.
The \emph{incidence graph} of $F$, denoted by $\ginc = (X \cup F, E)$ is the graph whose vertices are both the variables and the clauses of $F$ and which contains the edge $\{x,c\}$ for $x \in X$ and $c \in F$ if and only if $x \in \var{c}$.
Observe that $\ginc$ is bipartite. See \cref{fig:prop:priminc} for an example.


\begin{figure}
  \centering
  \includegraphics[page=1,width=5cm]{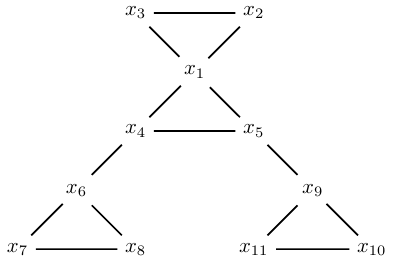} \hfill
  \includegraphics[page=2,width=8cm]{figs/graphcnf.pdf}
  \caption{The primal $\gprim$ and incidence $\ginc$ graphs of $F=C_1 \wedge C_2 \wedge C_3 \wedge C_4 \wedge C_5 \wedge C_6$ where $C_1 = (x_1 \vee x_2 \vee x_3), C_2 = (x_1 \vee x_4 \vee x_5), C_3 = (x_4 \vee x_6), C_4 = (\neg x_5 \vee x_9), C_5=(x_6 \vee x_7 \vee x_8), C_6=(x_9 \vee x_{10} \vee x_{11})$.}
  \label{fig:prop:priminc}
\end{figure}

\textbf{Treewidth.} We study the structure of $F$ by analyzing the structure of $\gprim$ or $\ginc$, using  the notion of \emph{treewidth}. A \emph{tree decomposition $\calT$ of a graph $G=(V,E)$} is a tree such that each node $t$ of $\calT$ is labeled by a subset $B_t$ of $V$, called a \emph{bag at node $t$}. Moreover, $\calT$ has the following properties: it is \textbf{connected}, that is, for every $x \in V$, the set $\{t \mid x \in B_t\}$ is connected in $\calT$ and it is \textbf{complete}, that is, for every edge $e$ of $G$, there exists a node $t$ such that $e \subseteq B_t$. \cref{fig:prop:tw} shows examples of tree decompositions.
The \emph{width of a tree decomposition $\calT$ of $G$}, denoted by $\tw{G}[\calT]$ is defined as $\max_{t\in\calT} |B_t|-1$ and the \emph{treewidth of $G$}, denoted by $\tw{G}$, is defined to be $\min_\calT \tw{G}[\calT]$, where $\calT$ runs over every valid tree decomposition of $G$.


\begin{figure}
  \centering
  \includegraphics[page=3,width=4cm]{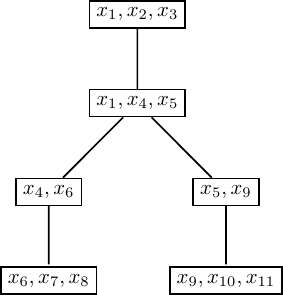}\hfill
  \includegraphics[page=4,width=4cm]{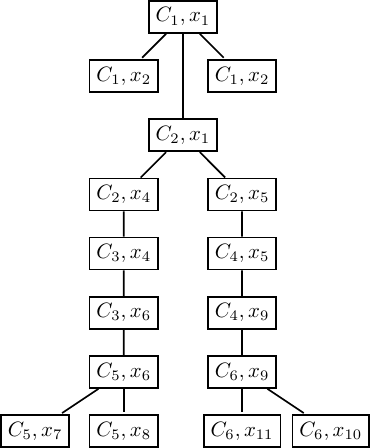}
  \caption{Tree decompositions for $\gprim$ (left) and $\ginc$ (right) from \cref{fig:prop:priminc}. }
  \label{fig:prop:tw}
\end{figure}

We apply the notion of treewidth to CNF formulas as follows. 
The \emph{primal treewidth $\ptw{F}$ of a CNF formula $F$} is defined as the treewidth of $\gprim$, while the \emph{incidence treewidth $\itw{F}$ of a CNF formula $F$} is the treewidth of $\ginc$.
It is not hard to see that for every CNF formula $F$, we have $\itw{F} \leq \ptw{F}+1$ and that for every $n \in \N$, there exists a CNF formula $F_n$ such that $\itw{F_n}=1$ and $\ptw{F_n}=n-1$.

\textbf{OBDD.} An \emph{Ordered Binary Decision Diagram (OBDD) over variables $X$} is a directed acyclic graph $C$ such that:
\begin{itemize}
\item Every node with outdegree $0$ is labeled by a constant $0$ or $1$ and is called a \emph{sink}.
\item Every other node is called a decision-node. It is labeled by a variable $x \in X$ and has two outgoing edges labeled by $0$ and $1$ respectively. We say that the decision-node \emph{tests} the variable $x$.
  \item $C$ has a unique node with indegree $0$ called the \emph{source}.
\end{itemize}
Moreover, there is an order $(x_1,\dots,x_n)$ on $X$ such that if $g$ is a decision-node on $x_i$, then every decision node that can be reached from $g$ by a path tests a variable $x_j$ with $j>i$. 

An OBDD $C$ over variables $X$ represents a Boolean function over variables $X$ as follows: an assignment $\tau \in 2^X$ satisfies $C$ if and only if there is a path $P = (g_0, \dots, g_k)$ from the source $g_0$ to a $1$-sink $g_k$ of $C$ such that for every $i < k$, the edge $(g_i,g_{i+1})$ is labeled by $\tau(x)$ where $x$ is the variable tested by $g_i$. 

\textbf{DNNF.} We assume the reader to be familiar with the notion of Boolean circuits, see~\cite{AroraB09} for details. A Boolean circuit $C$ is in \emph{Negation Normal Form (NNF)} if it only contains $\wedge$-gates and $\vee$-gates, and its inputs are labeled by literals. Given a gate $g$ of $C$, we denote by $\var{g}$ the set of variables appearing in the subcircuit rooted in $g$. We say that an $\wedge$-gate $g$ with inputs $g_1,\dots,g_k$ is \emph{decomposable} if and only if $\var{g_i} \cap \var{g_j} = \emptyset$ for every $i<j$. A \emph{Decomposable NNF (DNNF) circuit} is a circuit where every $\wedge$-gate is decomposable.
An $\vee$-gate $g$ with inputs $g_1,\dots,g_k$ is said to be \emph{deterministic} if and only if for every $i<j$, the models of $g_i$ and
$g_j$ are disjoint. In other words, $g$ is deterministic if for every model $\tau \in 2^{\var{g}}$ of $g$, there exists a unique $i \leq k$ such that $\tau$ is a model of $g_i$. A \emph{deterministic DNNF (d-DNNF) circuit} is a DNNF where every $\vee$-gate is deterministic. Observe that determinism is a semantic notion. It is actually coNP-complete to decide whether a given $\vee$-gate in a DNNF is deterministic.

In this paper, we are interested in a restriction of DNNF called \emph{structured DNNF (SDNNF)}~\cite{PipatsrisawatD08}. Structuredness is a syntactic restriction of the way an $\wedge$-gate can split variables in a DNNF. It is based on the notion of \emph{variable trees} (vtree for short): a vtree over $X$ is a rooted binary tree $T$ such that the leaves of $T$ are in one-to-one correspondence with $X$.  Given a node $t$ of $T$, we denote by $\var{t} \subseteq X$ the set of variables labeling the leaves of the subtree of $T$ rooted at $t$. Let $t$ be a node of $T$ with children $t_1,t_2$. Given an $\wedge$-gate $g$ with two inputs $g_1,g_2$, we say that $g$ \emph{respects} $t$ if and only if it has two inputs $g_1,g_2$ and $\var{g_1} \subseteq \var{t_1}$ and $\var{g_2} \subseteq \var{t_2}$. A DNNF circuit \emph{respects a vtree $T$} if for every $\wedge$-gate $g$ of $C$, there is a node $t$ of $T$ such that $g$ respects $t$. If a DNNF circuit $C$ respects a vtree $T$, we say that $C$ is a \emph{structured DNNF circuit}.

\textbf{SDD.} SDD~\cite{Darwiche11} is a restriction of structured deterministic DNNF enjoying more tractable operations and some form of canonicity (though the canonical circuit is not the minimal one in this case). Most proofs regarding SDD in this paper have been moved to the appendix, hence we leave out the technical definitions which can be found in the appendix. 

\section{Tree Decision Diagrams}
\label{sec:prop:dd}

Let $T$ be a vtree whose leaves are labeled by a set of variables $X$. A \emph{Non-deterministic Tree Decision Diagram (nTDD for short) $C=(N,E)$ respecting the vtree $T$}, is defined as follows: 
\begin{itemize}
\item $N = \biguplus_{t \in T} N_t$ is a set of nodes, partitioned into disjoint sets $N_t$ for each node $t$ of $T$. The elements of $N_t$ are called \emph{$t$-nodes}.
\item If $t$ is a leaf labeled by $x$, then every node in $N_t$ is labeled by either $x$, $\neg x$, $1$ or $0$.
\item $E$ maps every $t$-node $g$ to its \emph{inputs}: if $t$ is a leaf, then $E(g) = \emptyset$. Otherwise, if $t$ has children $t_1,t_2$, $E(g) \subseteq N_{t_1} \times N_{t_2}$, that is, $E(g)$ is a set of pairs $(g_1,g_2)$ such that $g_1 \in N_{t_1}$ is a $t_1$-node and $g_2 \in N_{t_2}$ is a $t_2$-node.
\item There is one distinguished $r$-node $\out$, called \emph{the output of $C$}, where $r$ is the root of $T$.
\end{itemize}

An nTDD $C$ computes a Boolean function over $X$ defined inductively as follows. Each $t$-node $g$ computes a Boolean function $f_g$ over variables $X_t$ where $X_t = \var{t}$:
\begin{itemize}
\item If $t$ is a leaf, then $g$ computes the Boolean function defined by its label: that is, if $g$ is labeled by $0$ then $f_g$ has no model, if $g$ is labeled by $1$ then every assignment of $x$ is a model of $f_g$, and if $g$ is labeled by $\ell \in \{x,\neg x\}$, the only model of $f_g$ is the assignment $\tau \in 2^{\{x\}}$ such that $\tau(\ell)=1$. 
\item If $t$ is an internal node with children $t_1,t_2$, then $\tau$ is a model of $f_g$ if and only if there exists $(g_1,g_2) \in E(g)$ such that $\tau|_{X_{t_1}}$ is a model of $f_{g_1}$ and $\tau|_{X_{t_2}}$ is a model of $f_{g_2}$. If $E(g)$ is empty, we make the convention that $f_g=\emptyset$ is the $0$ constant function.
\end{itemize}
An nTDD $C$ computes the Boolean function $f_C$ defined as $f_{\out}$, the function computed in its output. We often abuse notation and say a model of $g$ instead of a model of $f_g$.
Another way of defining $f_g$ is as $f_g = \bigvee_{(g_1,g_2) \in E(g)} (f_{g_1} \wedge f_{g_2})$. This definition allows us to see that an nTDD is just a structured DNNF written in a slightly different way. We chose this presentation however because it is more convenient to define TDDs. In fact, every structured DNNF can also be rewritten as an nTDD by smoothing the circuit and ensuring that $\wedge$-gates and $\vee$-gates alternate. The \emph{size} $|C|$ of nTDD $C=(N,E)$ is defined as $\sum_{n \in N} |E(n)|$. The \emph{width} of nTDD $C=(N,E)$ respecting vtree $T$ is defined as $\max_{t \in T} |N_t|$.  

  \cref{fig:prop:TDD} shows a vtree $T$ over variables $X = \{x_1,\dots,x_4\}$, an nTDD $C$ respecting $T$, and the interpretation of $C$ as a DNNF. Its width is $2$, and its size is $8$. We grouped together the set of $t$-nodes for every node $t$ of $T$. The assignment defined as $\tau(x)=0$ for every $x \in X$ is a model of $C$ because it is a model of every node pictured in red. 

\begin{figure}
  \centering
    \includegraphics[width=3cm]{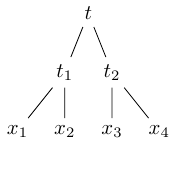}
    \includegraphics[width=6cm]{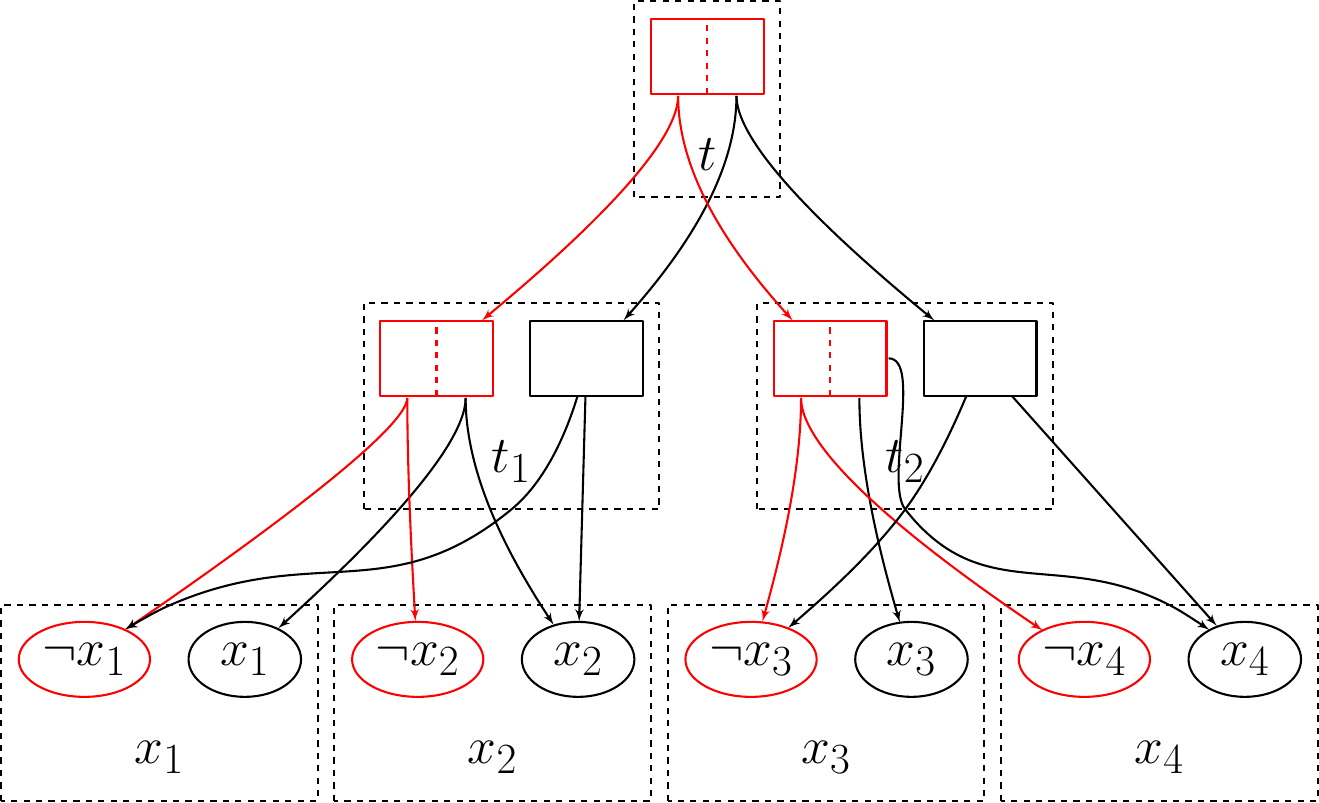}
    \includegraphics[width=4cm]{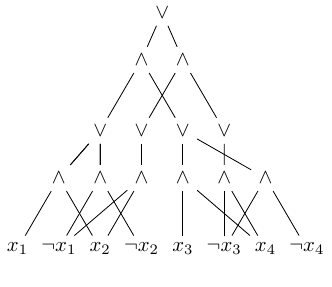}\hfill      
  \caption{A vtree,  an nTDD respecting it and the corresponding structured DNNF.}
  \label{fig:prop:TDD}
\end{figure}

Another way of characterizing the models of $C$ is via the notion of certificates. Given an assignment $\tau \in 2^X$,  a \emph{certificate for $\tau$ in $C$} is an nTDD $\calP$ formed by picking exactly one $t$-node $g^\calP_t$ of $C$ for every node $t$ of $T$ such that:
\begin{itemize}
\item If $t$ is a leaf of $T$, then $g^\calP_t$ is either labeled by $1$ or by a literal $\ell$ such that $\tau(\ell)=1$.
\item If $t$ is a node of $T$ with children $t_1,t_2$, then $(g^\calP_{t_1},g^\calP_{t_2}) \in E(g^\calP_t)$. 
\end{itemize}
The red part of \cref{fig:prop:TDD} represents the certificate for $\tau$, where $\tau$ is the assignment setting every variable to $0$, which is indeed a model of the circuit. More generally, a certificate for $\tau$ in $C$ is a witness of the fact that $\tau$ is a model of $C$:
\begin{proposition}[$\star$]
  \label{prop:prop:tdd-pt-models} Let $T$ be a vtree over $X$ and $C$ an nTDD respecting $T$. For every $\tau \in 2^X$, $\tau$ is a model of $C$ if and only if there exists a certificate $\calP$ for $\tau$ in $C$. In particular, for every node $t$ of $T$, $\tau|_{X_t}$ satisfies $g^\calP_t$.
\end{proposition}
\begin{toappendix}
\begin{proof}[Proof of \cref{prop:prop:tdd-pt-models}]
  First assume $\tau$ has a certificate $\calP$ in $C$. By induction, we prove that $\tau|_{X_t}$ is a model of $g_t^\calP$ for every node $t$ of $T$. It is true for the leaves of $T$ by definition of certificates. Now if $t$ is a node with children $t_1,t_2$, then $X_t = X_{t_1} \uplus X_{t_2}$ and $\tau|_{X_t} = \tau|_{X_{t_1}} \times \tau|_{X_{t_2}}$. By induction, $\tau|_{X_{t_i}}$ satisfies $g_{t_i}^\calP$ for $i \in \{1,2\}$. By definition of certificates, $(g_{t_1}^\calP, g_{t_2}^\calP) \in E(g^\calP_{t})$, hence $\tau|_{X_t}$ satisfies $g_t^\calP$.
  In particular, $\tau$ is a model of $g_r^\calP = \out$, hence $\tau$ is a model of $C$.

  Now let $\tau$ be a model of $C$. We construct a certificate. We let $g_r^\calP$ be $\out$. Now assume that we have constructed $\calP$ for every node above some node $t$ in $T$, including $t$. Let $t_1,t_2$ be the children of $t$. Now, since $\tau$ is a model of $C$, we know that there exists $(g_1,g_2) \in E(g_t^\calP)$ such that $\tau|_{X_{t_i}}$ is a model of $g_i$ for $i \in \{1,2\}$. We let $g_{t_1}^\calP = g_1$ and $g_{t_2}^\calP = g_2$. By definition, $(g_{t_1}^\calP, g_{t_2}^\calP) \in E(g_t^\calP)$ which corresponds to the definition of certificates.

  To show that this process constructs a certificate, it remains to show that it is correct on input nodes. But by definition, if $t$ is a leaf and hence $g_t^\calP$ is an input, $\tau$ is a model of $g_t^\calP$. Hence, either $g_t^\calP$ is labeled by $1$ or by a literal $\ell$ such that $\tau(\ell)=1$, which concludes the proof.
\end{proof}
\end{toappendix}



\textbf{Determinism.} A TDD $C=(N,E)$ is an nTDD respecting the following extra properties (which we will sometimes refer to as \emph{determinism}) for every node $t$ of $T$:
\begin{itemize}  
\item If $t$ is a leaf labeled by $x$, then no two nodes of $N_t$ can be satisfied simultaneously. Syntactically, this is the case if and only if $N_t$ contains at most one node labeled by $x$, at most one node labeled by $\neg x$ and at most one node labeled by $1$. Moreover, if there is a node labeled by $1$, then all other nodes of $N_t$ are labeled by $0$. 
\item For all distinct $g,g' \in N_t$, we have $E(g) \cap E(g') = \emptyset$.  That is,  every pair of nodes $(g_1,g_2)$ is the input of \emph{at most} one node.
\end{itemize}

Our notion of determinism is similar to others in the literature. First, it resembles the notion of determinism for bottom-up tree automata~\cite{tata} where a pair of states from children nodes gives at most one state in the parent node. Similar constructions have also been used in probabilistic circuits to guarantee determinism, see for example~\cite{shih2020probabilistic} and MDNets in~\cite{wang2023compositional}. 

Contrary to the notion of determinism for DNNF, the notion of determinism for TDD is syntactic. Therefore, it can be checked in polynomial time whether a given non-deterministic TDD respects the determinism property. Moreover, it induces a very strong form of determinism. We prove this with a bottom-up induction along the vtree.
\begin{theorem}[$\star$]
  \label{thm:prop:tdd-double-det}
  Let $C=(N,E)$ be a TDD respecting a vtree $T$. For every node $t$ of $T$ and $t$-nodes $g,g'$, $f_g$ and $f_{g'}$ have disjoint models. As a consequence, for every model $\tau$ of $C$, there exists a unique certificate $\calP_C(\tau)$ for $\tau$ in $C$. 
\end{theorem}
\begin{toappendix}
\begin{proof}[Proof of \cref{thm:prop:tdd-double-det}.]
  The proof is by induction on $T$. It is obvious if $t$ is a leaf by assumption. Indeed, either exactly one $t$-node is labeled by $1$ and all the other $t$-nodes are labeled by $0$. In this case the property is clear since there is a unique $t$-node having a model. Otherwise, we have at most one $t$-node labeled by $x$, at most one $t$-node labeled by $\neg x$ and all others labeled by $0$. Again, we have disjoint models.
  
  Now let $t$ be a node of $T$ with children $t_1,t_2$ and assume the property holds for $t_1$ and $t_2$. Let $g$ and $g'$ be two $t$-nodes. Let $\tau$ be a model of $g$ and let $(g_1,g_2) \in E(g)$ be such that $\tau|_{X_{t_1}}$ is a model of $g_1$ and $\tau|_{X_{t_2}}$ is a model of $g_2$. Now if $\tau$ is also a model of $g'$, there exists $(g'_1,g'_2) \in E(g')$
  such that $\tau|_{X_{t_1}}$ is a model of $g'_1$ and $\tau|_{X_{t_2}}$ is a model of $g'_2$. By determinism, $(g_1,g_2) \neq (g_1',g_2')$. Without loss of generality, assume $g_1 \neq g_1'$. In this case, $\tau|_{X_{t_1}}$ is a model of both $g_1$ and of $g_1'$. But this is not possible by induction. Hence $\tau$ is not a model of $g'$.

    We now show that there is a unique certificate for every model $\tau$ of $C$. Indeed, we have observed that if $\calP$ is a certificate for $\tau$ then $\tau|_{X_t}$ is a model of $g_t^\calP$. From what precedes, there is a unique $t$-node $\alpha_t$ that is satisfied by $\tau|_{X_t}$, hence $\alpha_t = g_t^\calP$ for every $t$, that is, $\calP$ is unique.
\end{proof}
\end{toappendix}

Observe that a TDD of width $k$ has size at most $2|X| \cdot k^2$. Indeed, $T$ has at most $2|X|$ nodes and each $t$-node can contain at most $k^2$ pairs.

Interestingly, we can construct the certificate for an assignment $\tau$ of $C$ efficiently in a bottom-up way, or report that $\tau$ is not a model of $C$. To do so, we select the unique leaf nodes satisfied by $\tau$ and construct the certificate bottom up by selecting the unique $t$-node whose input contains the pair $g_1,g_2$ of $t_1$-node and $t_2$-nodes inductively constructed so far. If no such node exists, we report that $\tau$ is not a model of $C$. Using appropriate data structures to represent the input of each $t$-node, we can find the right $t$-node in constant time. Hence we can construct a certificate for $\tau$ in time $O(|X|)$ if it exists.

A consequence of \cref{thm:prop:tdd-double-det} is that the DNNF interpretation of a TDD is deterministic, which proves that TDD is a subclass of structured d-DNNF:
\begin{theorem}[$\star$]
    \label{thm:prop:dtdd-vs-sddnnf} Given a TDD $C$ respecting a vtree $T$, one can construct a structured d-DNNF $C'$ respecting $T$ and computing the same function as $C$ in time $O(|C|)$. 
  \end{theorem}
  \begin{toappendix}
  \begin{proof}[Proof of \cref{thm:prop:dtdd-vs-sddnnf}.]
    Let $v_g$ be one $\vee$-gate corresponding to $t$-node $g$ in the DNNF corresponding to $C$. It is of the form $\bigvee_{(g_1,g_2) \in E(g)} (v_{g_1} \wedge v_{g_2})$. If $v$ is not deterministic, there exists $(g_1,g_2)$ and $(g_1',g_2')$ with $(g_1,g_2) \neq (g_1',g_2')$ and $\tau$ that is both a model of $(v_{g_1} \wedge v_{g_2})$ and of $(v_{g'_1} \wedge v_{g'_2})$. Assume wlog $g_1 \neq g_1'$. Then $\tau$ is a model of $g_1$ and of $g_1'$ which contradicts \cref{thm:prop:tdd-double-det}. 
  \end{proof}
\end{toappendix}

In particular, every tractable query for structured d-DNNF is also tractable for TDD. For example, we can efficiently compute the number of models of a TDD~\cite{DarwicheM2002}, enumerate them with delay $O(|X|)$~\cite{AmarilliBJM17} and so on. 

\section{Tractable Transformations}
\label{sec:transformations}

Since the publication of the knowledge compilation map~\cite{DarwicheM2002}, it is common in the field to compare newly introduced representations to others by analyzing them with respect to a set of standard queries and transformations. Since, due to \cref{thm:prop:dtdd-vs-sddnnf}, every TDD can be efficiently transformed into a deterministic, structured DNNF (d-SDNNF) and on those one can already perform all queries from~\cite{DarwicheM2002} efficiently~\cite{PipatsrisawatD08}, TDDs inherit all these efficient queries. So we here only focus on the transformations, showing that TDD allow more efficient transformations than d-SDNNF, and both canonical and general SDD. In fact, TDD allow the same efficient transformations as OBDD.

\begin{figure}[t]

  \begin{tabular}{l|l}

\textbf{Transformation Name}	& \textbf{Description} \\
	\hline
	Conditioning (CD) & given a variable $x$ and $a\in \{0,1\}$ compute representation for $f[x/a]$ \\
	Forgetting (FO) &  given a list $x_1, \ldots, x_\ell$ of variables compute $\exists x_1 \ldots \exists x_\ell \, f$\\
	Singleton Forgetting (SFO) &  same as FO, but only for a single variable\\
	Conjunction ($\land$C) &  compute representation for $\bigwedge_{i\in [\ell]} f_i$\\
	Bounded Conjunction ($\land$BC) &  same as $\land$C, but only two input representations\\
	Disjunction ($\lor$C) &  compute representation for $\bigvee_{i\in [\ell]} f_i$\\
	Bounded Disjunction ($\lor$BC)& same as $\lor$C, but only two input representations \\
	Negation ($\neg$C)& compute representation for $\neg f$ \\
\end{tabular}

\vspace{0.5cm}
\begin{tabular}{l|c|c|c|c|c|c|c|c|l}
			& CD           & FO          & SFO         & $\wedge$C    & $\wedge$BC   & $\vee$C     & $\vee$BC    & $\neg$C  & references     \\\hline
			TDD & $\checkmark$ & \textbullet  & $\checkmark$ & \textbullet & $\checkmark$ & \textbullet & $\checkmark$ & $\checkmark$ & this paper\\
			OBDD  & $\checkmark$ & \textbullet  & $\checkmark$ & \textbullet & $\checkmark$ & \textbullet & $\checkmark$ & $\checkmark$ & \cite{DarwicheM2002} \\
			SDD  & $\checkmark$ & \textbullet  & $\checkmark$ & \textbullet & $\checkmark$ & \textbullet & $\checkmark$ & $\checkmark$ & \cite{van2015role}\\
			canonical SDD  & \textbullet & \textbullet  & \textbullet & \textbullet & \textbullet & \textbullet & \textbullet & $\checkmark$ & \cite{van2015role}\\
			d-SDNNF  & $\checkmark$ & \textbullet & \textbullet & \textbullet  & $\checkmark$ & \textbullet & \textbullet & \textbullet & \cite{PipatsrisawatD08,Vinall-Smeeth24}\\
		\end{tabular}
                \caption{Overview of the transformations from the knowledge compilation map~\cite{DarwicheM2002} that can be performed efficiently on different representation languages. The first table describes the transformations. For all of them, either one input representation of a Boolean function $f$ or a list of representations of such functions $f_1, \ldots, f_\ell$ is given. Some transformations take additional inputs that are stated explicitly. In the second table, a $\checkmark$ means that the operation can be performed in polynomial time on representations from the language, whereas a \textbullet{} means that it takes super-polynomial time. All negative results are unconditionally true. For all transformations, we require that all inputs and outputs have the same vtree, resp.~variable order.}
	\label{tab:transformations}
\end{figure}

We give a compact description of the standard transformations in the first table of \cref{tab:transformations}; for additional discussion and justifications of these transformations see~\cite{DarwicheM2002}. The main result of this section is the following.

\begin{theorem}[$\star$]\label{thm:transformations}
	The efficient transformations that TDD allow are as described in \cref{tab:transformations}. 
\end{theorem}
The proof of \cref{thm:transformations} is not too hard but rather long and tedious, so we defer it to the appendix. We give some intuition here. Conditioning (CD) for a variable $x$ and $b \in \{0,1\}$ is obtained as usual in circuits, by replacing inputs labeled by $x$ with $b$ and inputs labeled by $\neg x$ with $1-b$. The important observation is to see that it preserves determinism: indeed, if $t$ is the node of the vtree labeled by $x$, and if there are inputs labeled by $x$ or $\neg x$, then we know that there is no $t$-node labeled by $1$. Then replacing inputs will create exactly one $t$-node labeled by $1$, which is consistent with the definition of determinism.

Bounded Conjunction ($\wedge$BC) is exactly the same algorithm as the one for structured d-DNNF circuits (see~\cite{PipatsrisawatD08}), and one just has to be careful that it preserves the syntactic properties of TDDs.
For negation ($\neg$C), the main idea is to first make the TDD \emph{complete}: if $C$ is a TDD respecting vtree $T$, we ensure that for every node $t$ of $T$ and assignment $\tau$ of $\var{t}$, there is exactly one $t$-node that is satisfied by $\tau$. This can be ensured bottom-up by creating a new $t$-node $n_t$ whose input is the list of pairs $(n_1,n_2)$ which are not inputs of any other $t$-node. In the end, if $r$ is the root of $T$, this creates an $r$-node which computes the negation of the TDD. 
The other transformations follow from those we have just described. For example, ($\vee$BC) follows from ($\neg$C) and ($\wedge$BC) since $f \vee g = \neg (\neg f \wedge \neg g)$. Similarly, SFO follows from ($\vee$BC) and (CD) since $\exists x.f = f[x/0] \vee f[x/1]$. 

\begin{toappendix}

\textbf{Elimination of constants.} Before studying transformations of TDD, we explain how one can remove constant inputs in TDD. This is necessary to properly explain some transformations such as conditioning or forgetting where some variables are removed from the vtree, which may induce a small ambiguity on the semantics of the circuit. 

Recall that by convention, a gate $g$ such that $E(g)=\emptyset$ computes the $0$-constant Boolean function. We refer to such gates as gates with empty inputs. We explain how one can remove $0$ constants and gates with empty inputs.

\begin{theorem}
  \label{thm:prop:zero-elim} Given a TDD $C$ respecting vtree $T$ with at least one model, one can build in linear time a TDD $C'$ respecting $T$, computing the same function as $C$ and such that $C'$ does not contain any $0$-labeled input nor gates with empty inputs. 
\end{theorem}
\begin{proof}
  Let $g$ be a gate that is either a $0$-labeled input or a gate with empty inputs. Assume that $g$ is a $u$-node for some node $u$ of $T$. If $u$ is the root of $T$, we simply remove $g$ from the circuit. Indeed, since $C$ has at least one model, $g$ is not the output of $C$ and it is not connected to any other gate, hence we can remove it without changing the function computed by the circuit.

  Otherwise, let $t$ be the parent of $u$ in $T$. We remove $g$ from the circuit and propagate the change upward: for every $t$-node $h$, we remove from $E(h)$ any pair containing $g$. Now for some $h$, $E(h)$ may become empty, in which case, we remove $h$ from the circuit similarly and propagate the change upward. We can apply this transformation until no $0$-input remains nor any gate with empty inputs. Observe that we only remove gates from $C$, hence the transformation preserves determinism.
\end{proof}

Removing $1$-inputs is however more delicate. Indeed, by construction, TDD have a rigid structure: if one node $h$ has a pair $(g,g')$ as input, with $g$ being a $1$-labeled input, then we cannot simply remove $g$ from the pair since it will change the definition of TDD. One way of doing it would be to allow TDD gates to have singletons in their input but it will make the definition of determinism slightly more delicate. We hence cannot always remove $1$-inputs in TDD.

There is a setting where it is possible. Assume that for some variable $x$, $C$ does not contain any input labeled by $x$ nor by $\neg x$. In this case, we say that $C$ \emph{does not syntactically depend on $x$}. Now, if $\ell$ is the leaf of $T$ labeled by $x$, every $\ell$-node of $C$ is labeled by a constant, and wlog by \cref{thm:prop:zero-elim}, we can assume these constants to be $1$. These constants can be removed from the circuit by removing $x$ from the vtree and renormalizing it along this new vtree.

Given a vtree $T$ over variables $X$ and $x \in X$, we let $T \setminus x$ be the vtree obtained as follows: we start by removing the leaf labeled by $x$ in $T$. Now the parent $t$ of this leaf in $T$ has only one child $u$ which is not allowed in a vtree. If $t$ is the root of the vtree, we simply remove $t$. Otherwise, we remove $t$ and connect $u$ directly to the parent $p$ of $t$ in $T$. See \cref{fig:prop:removing-leaves} for an illustration. 

\begin{figure}
  \centering
  \includegraphics[page=1]{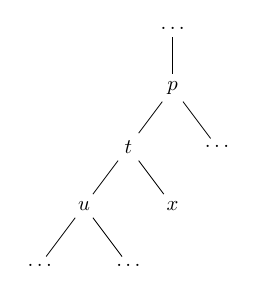}
  \includegraphics[page=2]{figs/vtree-minusx.pdf}
  \includegraphics[page=3]{figs/vtree-minusx.pdf}
  \caption{Removing $x$ from a vtree $T$: we first remove the $x$ leaf, then remove the parent of $x$ by plugging its only child $u$ with its parent $p$.}
  \label{fig:prop:removing-leaves}
\end{figure}

\begin{theorem}
  \label{thm:prop:remove-cst} Let $T$ be a vtree over $X$ and $C$ a TDD respecting $T$. Assume there is $x \in X$ such that $C$ does not syntactically depend on $x$. Then we can build a TDD $C'$ respecting $T \setminus x$ in linear time. 
\end{theorem}
\begin{proof}
  We let $l$ be the leaf of $T$ labeled by $x$ and use the same notation as in \cref{fig:prop:removing-leaves} for the other nodes of $T$. First assume that every $l$-node is labeled by $\bot$. In this case, every $t$-node $h$ has no model and then, the whole circuit has no model and we can replace it by a TDD computing $\bot$.

  Otherwise, there exists exactly one $l$-node $g$ labeled by $\top$, all the other $l$-nodes being labeled by $\bot$. Let $h$ be a $t$-node and let $U_h$ be the set of $u$-nodes $g'$ such that $(g',g) \in E(h)$. The models of $h$ are exactly the disjoint union of the models of $g'$, for $g' \in U_h$. For every $t$-node $h$, we introduce a new $u$-node $g_h$ whose inputs are $\bigcup_{v \in U_h} E(v)$ and remove every other $u$-node. It is clear that $g_h$ and $h$ have the same models over $X \setminus \{x\}$. 

Now if $t$ is the root of $T$, the output of $C$ is some $t$-node $h$. We remove every $t$-node from $C$ and choose $g_h$ as the new output. We hence have a new TDD over $T \setminus x$ which computes the same function as $C$ but over $X \setminus \{x\}$.

If $t$ is not the root of $T$, let $p$ be its parent in $T$. Let $v$ be a $p$-node of $C$. Its inputs are of the form $(h, b)$ with $h$ a $t$-node. We replace each such input of $v$ by $(g_h,b)$. It does not change the function computed by $v$ but now $g_h$ is a $u$-node and we have a new TDD over $T \setminus x$ which computes the same function as $C$ but over $X \setminus \{x\}$.

The transformation preserves determinism because if $C$ is deterministic, then $U_h \cap U_{h'} = \emptyset$ for distinct $t$-nodes $h,h'$. Hence, $E(g_h) \cap E(g_{h'}) = \emptyset$.
\end{proof}

\textbf{Conditioning.} Conditioning TDDs works similarly as the DNNF conditioning. Let $C$ be a TDD respecting some vtree $T$ over variables set $X$. Let $x \in X$ be a variable and $b \in \{0,1\}$. Let $\ell = x$ if $b=1$ and $\ell = \neg x$ otherwise. We let $C'$ be the circuit obtained by simply relabelling every input labeled by $\ell$ with $1$ and every input labeled by $\neg \ell$ with $0$. This new circuit almost computes $C$ conditioned by $\tup[x/b]$, but not exactly. Indeed, $C'$ is still defined over $X$. Its models are hence  $\{\tau \times \tup[x/b'] \mid b' \in \{0,1\} \text{ and } \tau \times \tup[x/b] \in f_C \}$ while what we would like to have is a circuit whose models are $\{ \tau \in 2^{X \setminus \{x\}} \mid \tau \times \tup[x/b] \in f_C\}$. We get such a circuit by observing that $C'$ does not syntactically depend on $x$. Hence, we can apply \cref{thm:prop:remove-cst} to get a TDD structured over $T \setminus x$ whose models are exactly $\{ \tau \in 2^{X \setminus \{x\}} \mid \tau \times \tup[x/b] \in f_C\}$. 

The transformation does not affect the size nor the determinism of the TDD. Indeed, assume that $C$ is deterministic. Since the structure of $C'$ is the same as the structure of $C$, we still have that for every pair $g,g'$ of internal $t$-nodes, $E(g) \cap E(g') = \emptyset$. For $t$ a leaf, observe that either there is a $t$-node labeled by $1$ in $C$, in which case, we have nothing to change since by determinism, no input is labeled by $\ell$ nor $\neg \ell$. Otherwise, we introduce at most one $1$-input and there is no $t$-node labeled by a literal anymore so the circuit is still deterministic. Finally, the transformation from \cref{thm:prop:remove-cst} also preserves determinism. We have proven:

\begin{theorem}
  \label{thm:prop:conditioning} Given a TDD $C$ of width $k$, a vtree $T$ over variables $X$,  $x \in X$ and $b \in \{0,1\}$, one can build in time $O(|C|)$  a TDD $C'$ with size at most $|C|$ and width at most $k$ computing $f_C[x/b]$. Moreover,  $C'$ respects $T \setminus x$.
\end{theorem}

\textbf{Negation.} TDD can be negated in polynomial time. This is a great advantage against structured deterministic DNNF, making them more akin to OBDD in this respect. Intuitively, we can transform a TDD such that we can add a new $t$-node that accepts every assignment of $X_t$ that is not accepted by another $t$-node. If we have this property at the root, then we can merge every $t$-node that is not $\out$ and set it as the new output to compute $\neg C$.

We say that a TDD $C=(N,E)$ respecting vtree $T$ is \emph{$t$-full} if $f_t := \bigvee_{g \in N_t} f_g$ is the $1$-constant Boolean function (every assignment $\tau \in 2^{X_t}$ is a model of $f_t$). It is \emph{full} if it is $t$-full for every node $t$ of $T$.

\begin{proposition}
  \label{prop:prop:tdd-full} Let $T$ be a vtree over $X$. For every TDD $C=(N,E)$ of width $k$ respecting $T$, one can build in time $O(k^2|X|)$  a full TDD $C'$ of width $k+1$ respecting $T$, computing the same function and of size at most $|C| + 2|X|\cdot k^2$. 
\end{proposition}
\begin{proof}
  The proof is by induction on $T$: for every node $t$, we transform $C$ in a bottom up way to make it $u$-full for every node $u$ below $t$ in $T$.

  If $t$ is a leaf, then either there is a $t$-node labeled by $1$ and it is obviously $t$-full. Otherwise, for every literal $\ell$ over $x$ such that no $t$-node is labeled by $\ell$, we add a $t$-node labeled by $\ell$. In the end, we have a $t$-node labeled by $x$ and one by $\neg x$, hence the circuit is now $t$-full.

  Now let $t$ be a node with children $t_1,t_2$ and assume that the circuit is $u$-full for any node below $t$ but $t$ itself. We transform the circuit into a $t$-full circuit. To this end, we add a new $t$-node $g$ to $C$ that is connected to every pair of $N_{t_1}\times N_{t_2}$ that is not plugged into any other $t$-node. In other words, we let $E(g) = (N_{t_1} \times N_{t_2}) \setminus \bigcup_{g' \in N_t} E(g')$. We now claim that the circuit is $t$-full. Indeed, let $\tau \in 2^{X_t}$ and let $\tau_1 = \tau|_{X_{t_1}}$, $\tau_2 = \tau|_{X_{t_2}}$. Since the circuit is $t_1$-full and $t_2$-full, we let $(g_1,g_2) \in N_{t_1}\times N_{t_2}$ such that $\tau_1$ is a model of $g_1$ and $\tau_2$ is a model of $g_2$. Let $g' \in N_t$ be the $t$-node such that $(g_1,g_2) \in E(g')$ which exists since every pair of $N_{t_1}\times N_{t_2}$ is now covered. We have that $\tau$ is a model of $g'$. Hence the circuit is now $t$-full.

  Observe that we have added at most one new gate per node of $t$ in the circuit hence we have increased the width of $C$ by at most $1$. The number of edges introduced for each node is at most $k^2$ and $T$ has at most $2|X|$ nodes, hence the total size of the circuit is now at most $|C|+2|X| \cdot k^2$.

  To build the circuit efficiently, one can do it as follows: for each $t$-node, create a $N_{t_1} \times N_{t_2}$ indexed matrix $M$ initialized to $0$. Loop over every $t$-node $v$ and their input. Upon seeing $(g_1,g_2) \in E(v)$, update $M[g_1,g_2]$ to $1$. The new gate to be added are exactly the entries of $M$ whose value is $0$ after having looped over every input of every $t$-node. The size of $M$ is at most $k^2$ and we loop over at most $k^2$ values to build it, hence we need $O(k^2)$ time to treat a $t$-node. Since there are at most $2|X|$ nodes in $T$, we need a total time $O(|X|k^2)$.
\end{proof}

Full TDD can easily be negated:
\begin{proposition}
  \label{prop:prop:tdd-neg} Let $T$ be a vtree over $X$. Given a full TDD $C=(N,E)$ of width $k$ respecting $T$, one can build in time $O(k)$ a full TDD $C'$ computing $\neg C$ of width at most $k$ and size at most $|C|$.
\end{proposition}
\begin{proof}
  Let $r$ be the root of $T$.  Since $C$ is full, it is $r$-full, that is, $\bigvee_{g \in N_r} f_g$ is the $1$-constant Boolean function over $C$. Moreover, by \cref{thm:prop:tdd-double-det}, this disjunction is deterministic. Hence, $\neg f_C = \bigvee_{g \in N_r, g \neq \out} f_g$. We build $C'$ by merging every $r$-node $g$ which is not $\out$: that is, we remove every $r$-node that is not the root and add a new $r$-node $g'$ such that $E(g') = \bigcup_{g \in N_r, g \neq \out} E(g)$. Obviously, $f_{g'} = \neg f_C$. We hence set $\out[C'] = g'$ and $C'$ computes $\neg f_C$. Moreover $C'$ is still full since we kept the previous output gate  and neither the width nor the size have increased. We only need $O(k)$ to perform this transformation as we only need to loop over every $r$-node but the output and there are at most $k-1$ such nodes. 
\end{proof}

Applying \cref{prop:prop:tdd-full} and \cref{prop:prop:tdd-neg} successively gives:
\begin{theorem}
  \label{thm:prop:tdd-neg} Let $T$ be a vtree over $X$. Given a TDD $C$ of width $k$ respecting $T$, there exists a TDD $C'$ of width at most $k+1$ and of size at most $|C|+2|X|\cdot k^2$ computing $\neg C$. Moreover, $C'$ can be constructed in time $O(k^2|X|)$. 
\end{theorem}

Observe that for full TDD, we can assume that we only have two $r$-nodes where $r$ is the root of $T$: one that is the output of $C$ and computes $f_C$, and one which computes the rest, hence $\neg f_C$. In other words, it resembles an OBDD where we have two sinks, one computing $f$ and the other $\neg f$. We will see in the next section that this connection can actually be made formal: an OBDD can be seen as a TDD if rooted at its sinks. 

\textbf{Apply.} We now prove that we can efficiently implement the apply operator over TDDs. The transformation for conjunction is mostly the same product construction as the one for OBDD or structured DNNF. For disjunction, we simply express $C \vee C'$ as $\neg (\neg C \wedge \neg C')$.

\begin{proposition}
  \label{prop:prop:tdd-and} Let $T$ be a vtree over variables $X$. Let $C_1$ and $C_2$ be two TDD respecting $T$ of width $k$ and $k'$ respectively. One can build a TDD $C$ respecting $T$ and computing $C_1 \wedge C_2$ in time $O(|C_1|\cdot|C_2|)$. The size of $C$ is at most $|C_1| \cdot |C_2|$ and its width is at most $k \cdot k'$. 
\end{proposition}
\begin{proof}
  We construct $C$ such that for every node $t$ of $T$, for every $t$-node $g_1$ of $C_1$ and $g_2$ of $C_2$, $C$ has a $t$-node $v_{\{g_1,g_2\}}$ which computes $f_{g_1}  \wedge f_{g_2}$. We proceed by induction on $T$.

  If $T$ only contains a leaf $t$ labeled by $x$, then for each pair $g_1,g_2$ of $t$-nodes in $C_1$ and $C_2$, we introduce a gate $v_{g_1,g_2}$ labeled by $e_1 \wedge e_2$ where $e_1$ is the label of $g_1$ and $e_2$ is the label of $g_2$. The label of $v_{g_1,g_2}$ is in $\{x,\neg x, 0, 1\}$ and $v_{g_1,g_2}$ clearly computes $f_{g_1} \wedge f_{g_2}$. Observe that since $C_1$ and $C_2$ are both deterministic, the resulting circuit is also deterministic. Indeed, assume $v_{g_1,g_2}$ is labeled by $1$, then it means that both $g_1$ and $g_2$ are labeled by $1$ and then every other $t$-node of $C_1$ and $C_2$ are labeled by $0$. Therefore, for every pair $g_1',g_2'$ distinct from $g_1,g_2$, $v_{g_1',g_2'}$ is labeled by $0$.  Otherwise, assume $v_{g_1,g_2}$ is labeled by literal $\ell$. Then either both $g_1,g_2$ are labeled by $\ell$, in which case $v_{g_1,g_2}$ is the only gate of $C$ labeled by $\ell$ because $C_1$ and $C_2$ do not have any gate labeled by $1$. Or $g_1$ is labeled by $1$ and $g_2$ by $\ell$, in which case again, $v_{g_1,g_2}$ is the only gate of $C$ labeled by $\ell$ because $C_1$ does not have any gate labeled by $\ell$ nor $\neg \ell$. The last case where $g_1$ is labeled by $\ell$ and $g_2$ by $1$ is symmetrical.

  Now let $t$ be a node of $T$ with children $t_1,t_2$. Assume that we have constructed $C$ up to $t_1$ and $t_2$. We construct the $t$-nodes of $C$ as follows: for every pair $g_1,g_2$ of $t$-nodes of $C_1$ and $C_2$ respectively, we create a $t$-node $v_{g_1,g_2}$ and define $E(v_{g_1,g_2}) = \{ (v_{a_1,b_1},v_{a_2,b_2}) \mid (a_1,a_2) \in E(g_1), (b_1,b_2) \in E(g_2) \}$.

  By induction, $v_{g_1,g_2}$ computes
  \begin{align*}
    \bigvee_{(a_1,a_2) \in E(g_1), (b_1,b_2) \in E(g_2)} (f_{a_1} \wedge f_{b_1}) & \wedge (f_{a_2} \wedge f_{b_2})   \\ & = \bigvee_{(a_1,a_2) \in E(g_1), (b_1,b_2) \in E(g_2)} (f_{a_1} \wedge f_{a_2}) \wedge (f_{b_1} \wedge f_{b_2}) \\
                                                                                  & = \bigvee_{(a_1,a_2) \in E(g_1)} (f_{a_1} \wedge f_{a_2}) \wedge \bigvee_{(b_1,b_2) \in E(g_2)} (f_{b_1} \wedge f_{b_2}) \\
    &= f_{g_1} \wedge f_{g_2}.
  \end{align*}

  It remains to show that $C$ is deterministic. Let $v_{a_1,b_1}$ be a $t_1$-node of $C$ and $v_{a_2,b_2}$ be a $t_2$-node of $C$. We claim that there is at most one $t$-node $v_{g_1,g_2}$  in $C$ such that $(v_{a_1,b_1}, v_{a_2,b_2}) \in E(v_{g_1,g_2})$. Indeed, there is at most one $t$-node $g_1$ of $C_1$ such that $(a_1,a_2) \in E(g_1)$ and at most one $t$-node $g_2$ of $C_2$ such that $(b_1,b_2) \in E(g_2)$. Hence $(v_{a_1,b_1}, v_{a_2,b_2})$ only appears in $E(v_{g_1,g_2})$, that is, $C$ is deterministic.

  By construction, it is clear that the width of $C$ is $k \cdot k'$ and of size at most $|C_1|\cdot|C_2|$. Constructing the circuit takes time $O(|C_1|\cdot|C_2|)$ since we loop once over every pair of inputs in $C_1, C_2$. 
\end{proof}

Combining \cref{prop:prop:tdd-and} and \cref{thm:prop:tdd-neg}, we can also construct a circuit computing $f_C \vee f_{C'}$ as $\neg (\neg C \wedge \neg C')$ when both $C$ and $C'$ are TDDs. In other words:

\begin{theorem}
  \label{thm:prop:tdd-apply} Let $T$ be a vtree over variables $X$. Let $C_1$ and $C_2$ be two TDDs respecting $T$ of width $k$ and $k'$ respectively. Let $f(x_1,x_2)$ be a Boolean function over $\{x_1,x_2\}$. There exists a TDD $C$ respecting $T$ and computing $f(C_1, C_2)$. Moreover, the size of $C$ is at most $O(|C_1| \cdot |C_2|)$, its width is at most $(k+1) \cdot (k'+1)$ and it can be built in time $O(|C_1| \cdot |C_2|)$.
\end{theorem}

Observe that one corollary of TDDs supporting the apply transformation is the fact that TDD is a complete language, that is, it can represent every Boolean function over a given set of variables $X$ and with any vtree over $X$. 
  \begin{proposition}
    \label{prop:prop:dTDD-complete} For every Boolean function $f$ over $X$ and vtree $T$ over $X$, there exists a TDD computing $f$.
  \end{proposition}
  \begin{proof}
    Let $T$ be a vtree over $X$. It is enough to observe that for any $\tau \in 2^X$, there exists a TDD $C_\tau$ respecting $T$ that accepts only $\tau$. This is straightforward by induction over $T$. Now, given a Boolean function $f$ with models $\tau_1,\dots,\tau_p$, we get a TDD $C_p$ for $f$ by constructing it iteratively as follows: $C_1=C_{\tau_1}$ and $C_{i+1} = APPLY(\vee, C_i, C_{\tau_{i+1}})$.
  \end{proof}

  \cref{prop:prop:dTDD-complete} may seem obvious but we observe that enforcing structure in dec-DNNF leads to an incomplete language. Indeed, if we fix a balanced vtree $T$ over variables $x_1,\dots,x_n$, then there is no decision-DNNF respecting $T$ and computing $x_1\vee \dots \vee x_n$. 

\textbf{Forgetting.} We conclude this section with one last transformation, which is notoriously easy for DNNF but not for deterministic DNNF. We have a similar property with TDDs. Given a Boolean function $f$ over $X$ and $Y \subseteq X$, we denote by $\exists Y. f$ the Boolean function whose models are $\{\tau|_{X \setminus Y} \mid \tau \models f \}$. This operation is often called \emph{forgetting} or \emph{existential projection}. Forgetting is easy in TDD but the transformation does not preserve determinism:

\begin{theorem}
  \label{prop:prop:TDD-forget} Let $T$ be a vtree over $X$ and $Y \subseteq X$. Given a TDD $C$ respecting $T$, one can construct in time $O(|C|)$ a non-deterministic TDD $C'$ respecting $T \setminus Y$ computing $\exists Y. f_C$. Moreover, $C'$ has smaller width and size than $C$.
\end{theorem}
\begin{proof}
  The same transformation as DNNF works: we construct $C'$ by relabeling in $C$ every input gate labeled by either $y$ or $\neg y$ for some $y \in Y$ with constant $1$ and removing the introduced constants as explained in \cref{thm:prop:remove-cst}.  It is easy to see by induction that $C'$ computes $\exists Y.f_C$
  with no increase in width nor size.
\end{proof}

The construction from \cref{prop:prop:TDD-forget} does not preserve determinism however because we could have more than one $1$-input gate in the inputs corresponding to $x$. 

\end{toappendix}

\section{Minimization and canonicity}
\label{sec:prop:minimization}

One of the most interesting features of TDD is that they can be minimized in polynomial time and that the minimal circuit is unique up to isomorphism, a property called \emph{canonicity}. The minimization algorithm is similar to the minimization for OBDD: we identify in the circuit pairs of gates that we call twins and which can be merged without changing the function computed by the circuits. We repeat this merging procedure until no twins can be found anymore. The circuit we obtain is then shown to be the minimal TDD computing the same Boolean function. 

We fix a vtree $T$ over variables $X$ and a TDD $C$ respecting $T$. Let $t_1$ be a node of $T$ that is not the root of $T$, let $t$ be its parent and $t_2$ its sibling. For a $t_1$-node $g_1$ and a $t$-node $g$, we define the \emph{siblings of $g_1$ with respect to $g$}, denoted by $\sib[g]$ to be $\{g_2 \mid (g_1,g_2) \in E(g)\}$, i.e., the set of $t_2$-nodes that appear together with $g_1$ in the inputs of $g$. 

We say that two $t_1$-nodes $g_1,g_1'$ are \emph{twins} if for every $t$-node $g$, we have $\sib[g] = \sib[g][g_1']$. For twins $g_1$ and $g_1'$, we define the \emph{twin contraction of $g_1,g_1'$} to be the operation where we replace $g_1,g_1'$ in $C$ by a new gate $v_{g_1,g_1'}$ such that $E(v_{g_1,g_1'}) = E(g_1) \cup E(g_1')$. Moreover, for any $t$-node $g$, we replace any pair of the form $(g_1,g_2)$ in $E(g)$ by $(v_{g_1,g_1'},g_2)$ and remove every pair of the form $(g_1',g_2)$. Observe that since $g_1$ and $g_1'$ are twins, $(g_1,g_2) \in E(g)$ if and only if $(g_1',g_2) \in E(g)$ by definition. Intuitively, two nodes are twins if the way they are used by the rest of the circuit is completely the same, hence contracting them does not change the function computed by the circuit.

\begin{lemma}[$\star$] \label{lem:contraction}
  After contracting a pair of twins, the function computed by a circuit is not changed. Moreover, the circuit is still a TDD.
\end{lemma}
\begin{toappendix}
\begin{proof}[Proof of \cref{lem:contraction}.]
Let $C'=(N',E')$ be the circuit obtained after contraction.  It is clear by definition that $v_{g_1,g_1'}$ in $C'$ computes $f_{g_1} \vee f_{g_1'}$. Since the twin operation does not change the number of $t$-nodes (only their input), there is a natural one-to-one correspondence between the $t$-nodes of $C$ and $C'$. We claim that they compute the same function. Indeed, every pair $(g_1,g_2)$ in $E(g)$ appears together with $(g_1',g_2) \in E(g)$ since $g_1$ and $g_1'$ are twins. This part of the function computes $(f_{g_1} \wedge f_{g_2}) \vee (f_{g'_1} \wedge f_{g_2})$ which is equal to $(f_{g_1} \vee f_{g_1'}) \wedge f_{g_2}$, that is, the function computed by $v_{g_1,g_1'}$. Hence, replacing pairs $(g_1,g_2)$ and $(g_1',g_2)$ in $E(g)$ with $(v_{g_1,g_1'},g_2)$ does not change the function $g$ computes. Determinism is preserved since we did not add duplicates.
\end{proof}
\end{toappendix}
We now define $m(C)$ to be the circuit obtained by the following transformation: first, if $r$ is the root of $T$, we remove every $r$-node but $\out$. We also remove every node that is not connected to the output of the circuit by a path. This does not change the function computed by $C$ since these gates are not used in any certificate. We then apply twin contraction to $C$ until no twins exist anymore. This process terminates since the number of nodes in $C$ decreases by $1$ with each contraction. Moreover, identifying and contracting twins can be done in polynomial time, hence we can construct $m(C)$ in polynomial time. We now prove that $m(C)$ is minimal and canonical by semantically characterizing the $t$-nodes of $m(C)$. 

We will describe the gates of $m(C)$ from the subfunctions they compute, which is similar to the description of canonical OBDD~\cite{sieling1993nc}.  A \emph{subfunction  of $f$ induced by $Y$},
or $Y$-subfunction for short, is a Boolean function over $X \setminus Y$ of the form $f[\tau]$ for some $\tau \in 2^{Y}$. Observe that $f$ has at most $2^{|Y|} \leq 2^{|X|}$ distinct $Y$-subfunctions, but it could have fewer. Indeed, two distinct assignments $\tau_1,\tau_2 \in 2^{Y}$ could be such that $f[\tau_1]$ and $f[\tau_2]$ have the same models over $2^{X \setminus Y}$, hence defining the same subfunction. A subfunction is said to be \emph{non-trivial} if it has at least one model. Given a vtree $T$ and a node $t$ of $T$, we will mostly be interested in the $X_t$-subfunctions of $f$. For example, consider the Boolean function $\mathit{PARITY}_X$ whose models are the assignments of $X$ having an even number of variables set to one and let $Y \subseteq X$. Then $\mathit{PARITY}_X$ has two $Y$-subfunctions: indeed, if $\tau \in 2^Y$ sets an even number of variables to $1$, then $\mathit{PARITY}_{X}[\tau] = \mathit{PARITY}_{X \setminus Y}$. Otherwise $\mathit{PARITY}_{X}[\tau] = \neg \mathit{PARITY}_{X \setminus Y}$. 


Now, we observe that a $t$-node in a TDD $C$ naturally defines an $X_t$-subfunction. Indeed, if $\tau_1,\tau_2$ are two models of a $t$-node $g$ and $\tau$ is a model of $C$ such that $\tau|_{\var{g}} = \tau_1$, then we can change the value of $\tau$ over $\var{g}$ to $\tau_2$, and it remains a model of $C$ because we only change the part of the certificate of $\tau$ below $g$. Hence, we have:

\begin{lemma}[$\star$]
  \label{lem:prop:gate-subfunction}
For a vtree node $t$ of $T$ and $g$ a $t$-node of $C$, let $\tau_1,\tau_2$ be two models of $g$. We have that $f_C[\tau_1]$ and $f_C[\tau_2]$ define the same $X_t$-subfunction, denoted by $\sub_g$. Moreover, for every model $\tau$ of $C$ such that $g$ is in the certificate of $\tau$, $\tau|_{X \setminus X_t}$ is a model of $\sub_g$.
\end{lemma}
\begin{toappendix}
\begin{proof}[Proof of \cref{lem:prop:gate-subfunction}.]
Let $\sigma$ be a model of $f_C[\tau_1]$. In particular, $\sigma \times \tau_1$ is a model of $C$. The $t$-node in the unique certificate of $\sigma \times \tau_1$ is $g$, since $\tau_1$ is a model of $g$. Now, we can swap the part of the certificate below $t$ with the certificate of $\tau_2$ below $t$. This way, we get a certificate for $\sigma \times \tau_2$, hence $\sigma \times \tau_2$ is a model of $f_C$, and then $\sigma$ is a model of $f_C[\tau_2]$. Symmetrically, we have that every model $\sigma$ of $f_C[\tau_2]$ is a model of $f_C[\tau_1]$. Hence $f_C[\tau_1]$ and $f_C[\tau_2]$ have the same models and define the same $X_t$-subfunction.  Finally, observe that for every model $\tau$ of $C$ whose certificate contains $g$, we have that $\tau|_{X \setminus X_t}$ is a model of $f_C[\tau|_{X_t}]$ and by what precedes, $f_C[\tau|_{X_t}]$ is $\sub_g$. 
\end{proof}
\end{toappendix}

By \cref{lem:prop:gate-subfunction}, we can map  every $t$-node $g$ of $C$ to an $X_t$-subfunction $\sub_g$ of $f_C$ defined as $f_C[\tau]$ for some arbitrary model $\tau$ of $g$. This directly gives a lower bound on the number of $t$-nodes in a TDD  representing a Boolean function $f$: it must be at least the number of non-trivial $X_t$-subfunctions of $f_C$. Indeed, if $\tau_1,\tau_2 \in 2^{X_t}$ are such that $f_C[\tau_1]$ and $f_C[\tau_2]$ define two distinct $X_t$-subfunctions, then they cannot be models of the same $t$-node. Now, if  $f_C[\tau_1]$ is non-trivial, then there must be a $t$-node $g_1$ such that $\tau_1$ is a model of $g_1$, since there exists at least one $\sigma \in 2^{X \setminus X_t}$ such that $\sigma \times \tau_1$ is a model of $C$. Similarly, if $f_C[\tau_2]$ is non-trivial, there is a $t$-node $g_2$ such that $\tau_2$ is a model of $g_2$. Since $\sub_{g_1} \neq \sub_{g_2}$, we have at least one gate per non-trivial $X_t$-subfunction of $f_C$.

\begin{theorem}
  \label{thm:prop:tdd-lowerbound} Given a Boolean function $f$ over variables $X$ and a vtree $T$ over $X$, the smallest TDD computing $f$ has at least $S_t$ $t$-nodes for every node $t$ of $T$ where $S_t$ is the number of non-trivial $X_t$-subfunctions of $f$.
\end{theorem}

The following proves that $m(C)$ matches the lower bound from \cref{thm:prop:tdd-lowerbound}. The proof boils down to showing that if there are more than $S_t$ number of $t$-nodes, then by the pigeonhole principle, two $t$-nodes must be mapped to the same subfunction and thus can be merged. If $t$ is the shallowest node where it happens, we can show that such $t$-nodes must be twins.

\begin{theorem}[$\star$]
  \label{thm:prop:tdd-min} Let $T$ be a vtree over $X$ and $C$ a TDD. Then $m(C)$ has exactly $S_t$ $t$-nodes, where $S_t$ is the number of non-trivial $X_t$-subfunctions of $f_C$. 
\end{theorem}
\begin{toappendix}
\begin{proof}[Proof of \cref{thm:prop:tdd-min}]
Assume $m(C)$ does not match the lower bound. It follows by the pigeonhole principle that there is a node $t$ of $T$ and two distinct $t$-nodes $g,g'$ such that $\sub_g = \sub_{g'}$ and $\sub_g$ is not trivial.  We can assume that $t$ is not the root $r$ of $T$ since $m(C)$ has only one $r$-node. Moreover, we can assume that for every ancestor $u$ of $t$ in $T$, there is exactly one $u$-node per $X_u$-subfunction. Indeed, if this is not the case, then we could have picked $u$ instead of $t$.

To summarize, there is a node $t_1$, with parent $t$ and sibling $t_2$, and two $t_1$-nodes $g,g'$ such that $\sub_g=\sub_{g'}$. Moreover, there is exactly $S_t$ $t$-nodes, one for each $X_t$-subfunction of $f_C$. We claim that in this case, $g$ and $g'$ are twins, contradicting the fact that $m(C)$ does not contain twins. Indeed, let $h$ be a $t$-node such that $E(h)$ contains a pair $(g,\gamma)$ for some $t_2$-node $\gamma$. We show that $E(h)$ also contains $(g',\gamma)$. Consider a model $\tau$ of $g$, a model $\tau'$ of $g'$ and $\sigma$ a model of $\gamma$. By definition, $\tau \times \sigma$ is a model of $h$. Let $\alpha$ be a model of $\sub_h$, that is, $\beta := \alpha \times \tau \times \sigma$ is a model of $C$. Now it follows that $\alpha \times \sigma$ is a model of $\sub_g$, hence, it is also a model of $\sub_{g'}$. In other words, $\beta' := \alpha \times \sigma \times \tau'$ is a model of $C$. Now the certificate of $\beta'$ must contain $h$, because $\alpha$ is a model of $\sub_h$ and we assumed that $h$ is the only $t$-node for this $X_t$-subfunction. Hence, $(g',\gamma) \in E(h)$. Since the argument is symmetric in $g$ and $g'$, we immediately get $\sib[h][g]=\sib[h][g']$ for any $t$-node $h$. In other words, $g$ and $g'$ are twins which contradicts the construction of $C$.

Hence, for every node $t$ of $T$ and $g,g'$ distinct $t$-nodes of $m(C)$, $\sub_g \neq \sub_{g'}$. The $X_t$-subfunction defined by $g$ is non trivial because we removed dangling nodes. Hence $m(C)$ has exactly $S_t$ node for every node $t$ of $T$. 
\end{proof}
\end{toappendix}

\cref{thm:prop:tdd-lowerbound,thm:prop:tdd-min} together prove that $m(C)$ has minimal size. Moreover, this minimal circuit is unique because each gate is uniquely defined by the $X_t$-subfunction it computes. TDD can therefore be minimized in polynomial time to a canonical minimal circuit.  The time needed to compute $m(C)$ is polynomial in the width $k$ of $C$ and linear in the number of variables. Indeed, removing non-accessible nodes can be done in time linear in $|C| \leq k^2|X|$. Contracting twins in a $t$-node can be done in time polynomial in the number of $t$-nodes, that is, in polynomial time in $k$, and we have to do it for every node $t$ of $T$ and there are at most $2|X|$ such nodes. The exact complexity of the minimization depends on the data structures used to represent $t$-nodes and their inputs. We leave fine-grained analysis for practical implementations.
\begin{theorem}
  \label{thm:prop:tdd-minization-complexity} Given a TDD $C$ of width $k$ over variables $X$, we can compute a minimal canonical representation $m(C)$ of $C$ in time $\poly(k) \cdot |X|$. 
\end{theorem}


\nocite{bojanczyk} 

{\bf Learnability. } An interesting application of canonicity is that  it allows us to design efficient $L^*$-style learning for TDD, as for finite automata and OBDD~\cite{angluin1987learning}. This result is framed in the {\em Minimally Adequate Teacher} model: there is a hidden Boolean function $f:2^X\to\{0,1\}$ which the learning agent can only access via two types of queries: {\em membership queries}, in which it tests some assignment on $X$, and an oracle answers whether it is a model or not; and {\em equivalence queries}, in which the agent tests a TDD, and an oracle answers whether it represents $f$, and in the negative case provides a counterexample. The goal of the process is to construct a minimal size TDD for $f$ with a low number of queries.

\begin{proposition}[$\star$]  \label{prop:learn}
Fix a vtree $T$ on $X$.
	There is an algorithm that learns the canonical TDD $C$ respecting $T$ in polynomial time in $|C|$ and with a polynomial number of oracle calls to membership and equivalence queries.
\end{proposition}
\begin{toappendix}
\begin{proof}[Proof of Proposition~\ref{prop:learn}]
We describe our algorithm as an adaptation of Angluin's $L^*$ original algorithm~\cite{angluin1987learning}.
We also borrow some notation and insights on Bojanczyk's version of it~\cite{bojanczyk}.

To simplify the algorithm, we will state the result for an algorithm that learns a minimal {\em full} TDD, as defined in the appendix of Section~\ref{sec:transformations}. We state without proof that a slight rewriting of Theorem~\ref{thm:prop:tdd-min}, by allowing trivial subfunctions, holds for full TDD.
For the final algorithm that learns a minimal TDD, we keep the algorithm as is, except for an extra step which trims the TDD.

The setting of the algorithm is this:
for every $t\in T$, we keep a set ${\sf Nodes}_t$ of partial assignments on $X_t$; and for every non-leaf $t\in T$ we keep a set ${\sf Tests}_t$ of partial assignments on $X\setminus X_t$. 
For a step of the algorithm, we call the current sets ${\sf Nodes}_t$ and ${\sf Tests}_t$ its {\em state}.

Given a non-leaf $t\in T$, and a given ${\sf Tests}_t$,  we say that partial assignments $\tau_1$ and $\tau_2$ on $X_t$ are {\em equivalent at $t$} if for every $\kappa\in {\sf Tests}_t$ it holds that $f(\tau_1 \times \kappa) = f(\tau_2 \times \kappa)$.
We note that two partial assignments $\tau_1$ and $\tau_2$ on $X_t$ can be tested for at-$t$-equivalence with $|{\sf Tests}_t|$ membership queries.

We say that the state of the algorithm is {\em closed} if, for every non-leaf node $t$, for any choice of $\tau_{\ell} \in {\sf Nodes}_{\ell(t)}$ and $\tau_r \in {\sf Nodes}_{r(t)}$, the partial assignment $\tau_\ell \times \tau_r$ is equivalent at $t$ to some $\tau \in {\sf Nodes}_t$. 

We say that the state of the algorithm is {\em left-consistent} if, for every non-leaf node $t$, for any $\tau_{1}, \tau_{2}  \in {\sf Nodes}_{\ell(t)}$ that are equivalent at $\ell(t)$ it holds that for every $\tau_r\in {\sf Nodes}_{r(t)}$ the partial assignments $\tau_{1} \times \tau_r$ and $\tau_{2} \times \tau_r$  are equivalent at $t$. 
We say the sets are {\em right-consistent} with an analogous definition, and we say the state is {\em consistent} if it is both left- and right-consistent.

Fix a non-leaf $t\in T$. We denote the class of partial assignments that are equivalent at $t$ to a $\tau\in{\sf Nodes}_t$ by $[\tau]$.

\begin{remark}\label{rem:learn:build}
If the state of the algorithm is closed and consistent, we can build a full TDD $C = (N, E)$ as follows: for each $t\in T$ define $N_t = \{[\tau]\mid \tau\in {\sf Nodes}_t\}$, and its root is $[\tau]$ for any $\tau\in{\sf Nodes}_{{\sf root}(T)}$ that evaluates to 1; and the upwards function $E$ is defined by $E([\tau_\ell], [\tau_r]) = [\tau_\ell \times \tau_r]$ for every $\tau_\ell\in{\sf Nodes}_{\ell(t)}$ and $\tau_r\in{\sf Nodes}_{r(t)}$. 
We prove that the construction of $E$ is correct: if the state of the algorithm is closed, then the latter $[\tau_\ell \times \tau_r]$ exists; and if it is consistent, then for every $\tau_\ell'\in [\tau_\ell]$ and $\tau_r'\in[\tau_r]$ it holds that $[\tau_\ell' \times \tau_r'] = [\tau_\ell \times \tau_r]$.
\end{remark}

{\bf The algorithm. } We begin by setting ${\sf Nodes}_t$ to $\emptyset$ for every non-leaf $t\in T$, except for $t = {\sf root}(T)$ where ${\sf Nodes}_t$ is set as a singleton of the empty assignment from $\emptyset$ to $\{0,1\}$.
Then:
\begin{enumerate}
\item Check for closure; if the current state is not closed, then, for the $t\in T$ in which the test failed, add the witnessing partial assignment to ${\sf Nodes}_t$. Repeat this until the state is closed.
\item Check for consistency; w.l.o.g. if the current state is not left-consistent, then, for the $t\in T$ in which the test failed, add the $\tau_r\in {\sf Nodes}_{r(t)}$ that witnesses failure to ${\sf Tests}_{\ell(t)}$, and go back to step 1.
\item The state of the algorithm is now closed and consistent, so we build the TDD from Remark~\ref{rem:learn:build} and run an equivalence query. If the query answers {\em true}, the algorithm ends and we return the TDD; otherwise, we receive a counterexample assignment $\tau_c$, and then for each non-leaf node $t\in T$ we add $\tau_c\vert_{X_t}$ to ${\sf Nodes}_t$, and then go back to step 1.
\end{enumerate} 
This ends the description of the algorithm.

Clearly, if the algorithm ends, then the TDD given as output is equivalent to the hidden function $f$. We argue that it is also minimal.
\begin{claim}\label{claim:learn:min}
The TDD produced by the algorithm is minimal for full TDD.
\end{claim}
\begin{proof}
For a non-leaf $t\in T$, if two $\tau_1,\tau_2\in{\sf Nodes}_t$ are not equivalent at $t$, then there is a partial assignment $\kappa\in{\sf Tests}_t$ that witnesses that $f[\tau_1]$ and $f[\tau_2]$ define different $X_t$-subfunctions. Since we have at most one node per $X_t$-subfunction of $f$, the TDD is minimal for full TDD by virtue of the stated rewriting of Theorem~\ref{thm:prop:tdd-min}.
\end{proof}

Yet we still need to prove that the algorithm actually ends. We use the following.
\begin{remark}\label{rem:learn:equivs}
The only way that equivalence classes inside each ${\sf Nodes}_t$ can change are (1) by adding more elements to ${\sf Tests}_t$, and this may separate a class into more than one class, and (2) by adding an element to ${\sf Nodes}_t$ that does not belong to any other existing class. In particular, the number of equivalence classes never decreases; and if their number stays the same, the equivalence classes do not change either.
\end{remark}

The fact that the algorithm ends follows straightforwardly by these two claims:

\begin{claim}\label{claim:learn:corr1}
If the equivalence query in step 3 answers {\em false}, then after adding the counterexample $\tau_c$ and its partial assignments to the ${\sf Nodes}_t$ sets, the next iteration of steps 1 and 2 strictly increases the number of equivalence classes under at-$t$-equivalence for at least one $t\in T$. 
\end{claim}
\begin{proof}
Let us suppose that after adding the partial assignments and going through steps 1 and 2, the equivalence classes at each $t\in T$ do not increase in number; and given Remark~\ref{rem:learn:equivs}, this is the same as them staying unchanged. Consider the following two TDDs that can be built by the current state, (1) by using the same partial assignments as the one we tested in the previous iteration of step 3, and (2) by using the partial assignments derived from $\tau_c$. By Remark~\ref{rem:learn:build}, these TDDs are isomorphic. We note that in this TDD, evaluating the assignment $\tau_c$ satisfies node $[\tau_c\vert_{X_t}]$ for every $t\in T$, and in particular $[\tau_c]$, which contradicts the fact that $\tau_c$ was a counterexample.
\end{proof}
Since equivalence queries are bounded by Claim~\ref{claim:learn:corr1}, the proof concludes with the following statement.
\begin{claim}\label{claim:learn:corr2}
The total number of membership queries is polynomial in the size of the target TDD $C^*$.
\end{claim}
\begin{proof}
Note that the only moment at which we add some element to any ${\sf Nodes}_t$ are steps 1 and 3. In the former, adding an element creates a new equivalence class, and the latter is done a total number of times that is bounded by $|C^*|$ by Claim~\ref{claim:learn:corr1}. We obtain that the size of ${\sf Nodes}_t$ for any $t\in T$ is at most $2\cdot|C^*|$.  

Similarly, the only moment in which we add some element to any ${\sf Tests}_t$ is step 2.
For a given $t$, we will only add elements from ${\sf Nodes}_{\ell(t)}$ to ${\sf Tests}_{r(t)}$, and only elements from ${\sf Nodes}_{r(t)}$ to ${\sf Tests}_{\ell(t)}$, so the size of any ${\sf Tests}_t$ is also bounded by $2\cdot|C^*|$.

From the arguments above, it follows that any closure check and any consistency check takes a number of membership queries that is polynomial in $|C^*|$.

Now, let us reason how many times step 1 and step 2 can occur. We note that every time both the closure check and the consistency check succeed, then step 3 is reached, and by Claim~\ref{claim:learn:corr1}, the number of times this occurs is bounded by $|C^*|$. Now, let us count the number of times there can be an iteration in which any of them fails. We note that every time the closure check fails, a new equivalence class is added, and every time the consistency check fails, an element is added to some ${\sf Tests}_t$. Therefore, the number of iterations with any of the failures is at most $|C^*| + \sum_{t\in T}|{\sf Tests}_t|$. In total, this proves that the number of total closure and consistency checks is also polynomial in $|C^*|$.
\end{proof}

It can easily be checked that every other step performed by the algorithm can be done in polynomial time.
\end{proof}
\end{toappendix}



\begin{toappendix}
\section{Determinizing TDD} 
\label{sec:prop:detTDD}

In this section, we show that non-deterministic TDD can be made deterministic by paying an exponential blow up in their width. The existence of such a TDD can be seen through subfunctions. Indeed, fix a vtree $T$ over variables set $X$ and a TDD $C$ respecting $T$.  For every node $t$ of $T$, we let $N_t$ be the set of $t$-nodes of $C$. Recall that by assumption, $|N_t| \leq k$. Given $S \subseteq N_t$, we let $f^t_S$ be the Boolean function whose models are the assignments $\tau \in 2^{X_t}$ such that:
\begin{itemize}
	\item for every $v \in S$, $\tau$ is a model of $v$,
	\item for every $v \in N_t \setminus S$, $\tau$ is not a model of $v$.
\end{itemize}
Given an assignment $\tau \in 2^{X_t}$, we can define the \emph{$t$-shape $S_t(\tau)$ of $\tau$} (simply called \emph{shape of $\tau$} when $t$ is clear from context) as the set $S_t(\tau) \subseteq N_t$ containing every $g \in N_t$
such that $\tau$ is a model of $g$.  Clearly $\tau$ is a model of $f^t_S$ if and only if $S = S_t(\tau)$. Moreover, it is not hard to prove by adapting the proof of \cref{lem:prop:gate-subfunction} that if  $S(\tau_1) = S(\tau_2)$, then  $f[\tau_1]$ and $f[\tau_2]$ define the same $X_t$-subfunction. This hints at the fact that $C$ cannot have more than $2^{|N_t|} \leq 2^k$ $X_t$-subfunctions and that they can be defined (possibly with redundancies) by the $f^t_S$ functions. Hence, since TDDs are complete and by \cref{thm:prop:tdd-min}, there must exist a TDD $C'$ computing the same function as $C$ and of width at most $2^k$.

The previous sketch does not, however, give a way of actually constructing the TDD for $C$. Below, we give a constructive proof allowing us to efficiently construct $C'$ from $C$. Before going into the main proof, we do a few key observations. For a set $S \subseteq N_t$, if there exists some $\tau \in 2^{X_t}$ such that $S = S^t(\tau)$, we say that $S$ is a $t$-shape of $C$ (without the need to explicitly refer to $\tau$ anymore). Let $t$ be an internal node of $T$ with children $t_1,t_2$ and let $S_1$ be a $t_1$-shape and $S_2$ be a $t_2$-shape. We say that $(S_1,S_2)$ \emph{generates $t$-shape $S$} if there exists $\tau_1 \in 2^{X_{t_1}}$ and $\tau_2 \in 2^{X_{t_2}}$ such that $S^{t_1}(\tau_1) = S_1, S^{t_2}(\tau_2) = S_2$ and $S^{t}(\tau_1 \times \tau_2) = S$. The following lemma shows that $(S_1,S_2)$ generates exactly one $t$-shape, and this shape is easy to compute from $(S_1,S_2)$:

\begin{lemma}
	\label{lem:prop:tshapes-compute}
	For every $t_1$-shape $S_1$ and $t_2$-shape $S_2$,  there exists a unique $t$-shape $S$ that is generated by $S_1$ and $S_2$ defined as $S := \{ g \in N_t \mid \exists (g_1,g_2) \in E(g) \text{ such that } g_1 \in S_1, g_2 \in S_2 \}$.
\end{lemma}
\begin{proof}
	Let $\tau_1 \in 2^{X_{t_1}}$ with $t_1$-shape $S_1$ and $\tau_2 \in 2^{X_{t_2}}$ with $t_2$-shape $S_2$ and let $S'$ be the $t$-shape of $\tau := \tau_1 \times \tau_2$. We want to show that $S'=S$. Let $g \in S$ and let $(g_1,g_2) \in E(g)$ with $g_1\in S_1, g_2 \in S_2$ which exists by definition of $S$. Since $g_1$ is in $S_1$ and $\tau_1$ has  shape  $S_1$, $\tau_1$ is a model of $g_1$. Similarly, $\tau_2$ is a model of $g_2$. Hence, $\tau$ is a model of $g$, that is, $g \in S'$.
	
	Now let $g \in S'$. By definition, $\tau$ is a model of $g$, that is, there exists $(g_1,g_2) \in E(g)$ such that $\tau_1$ is a model of $g_1$ and $\tau_2$ is a model of $g_2$. But it then means that $g_1 \in S_1$ and $g_2 \in S_2$ since $S_1$ and $S_2$ are the shapes of $\tau_1$ and $\tau_2$ respectively. In other words, $g \in S$ and we have proven $S'=S$.
\end{proof}

The next lemma follows directly from definitions and from \cref{lem:prop:tshapes-compute} but it nicely wraps up what will be needed later.

\begin{lemma}
	\label{lem:prop:tshapes-complete}
	
	Let $\tau \in 2^{X_t}$ be an assignment with $t$-shape $S$. Let $\tau_1 := \tau|_{X_{t_1}}$, $\tau_2 := \tau|_{X_{t_2}}$, and let $S_1,S_2$ be their respective shape. Then $(S_1,S_2)$ generates $S$. Moreover, if $(S_1,S_2)$ generates $S$ then for every $\tau_1 \in 2^{X_{t_1}}$ with shape $S_1$ and $\tau_2 \in 2^{X_{t_2}}$ with shape $S_2$, $\tau_1 \times \tau_2$ has shape $S$. 
\end{lemma}
\begin{proof}
	The first part follows directly from the definition of $(S_1,S_2)$ generates $S$ since $\tau = \tau_1 \times \tau_2$. The second part follows from \cref{lem:prop:tshapes-compute}. Indeed,  $S$ is the unique $t$-shape generated by $(S_1,S_2)$ and hence $\tau_1 \times \tau_2$ have shape $S$ by definition.
\end{proof}

We are now ready to prove the main theorem of this section:
\begin{theorem}
	\label{thm:prop:detTDD} Let $T$ be a vtree over variables set $X$. Let $C$ be a non-deterministic TDD of width $k$. One can construct a full TDD $C'$ of width at most $2^k$ computing the same function as $C$ in time $|C| \cdot 2^{O(k)}$.
\end{theorem}
\begin{proof}
	As before we let $N_t$ be the set of $t$-nodes and for $S \subseteq N_t$, we let $f_S$ be the function whose models are assignments that satisfy exactly $S$.  We build $C'$ inductively in a bottom-up fashion. We maintain the following invariant: for every node $t$ of $T$ and $S \subseteq N_t$ such that $f_S \neq \emptyset$, $C'$ contains a $t$-node $v^t_S$ computing $f_S$.
	
	If $t$ is a leaf of $T$ labeled by $x$, then there are exactly two possible assignments: the one mapping $x$ to $1$ and the other mapping $x$ to $0$. If both assignments satisfy the same set $S \subseteq N_t$ of $t$-nodes, it means that $N_t$ only contains inputs labeled by $1$ or $0$. Hence we have exactly one $v^S_t$ node labeled by constant $1$. Otherwise, we have two possible sets $S_0,S_1 \subseteq N_t$, the one containing every $t$-node satisfied by $\tau_0 := \tup[x/0]$ and the one containing every $t$-node satisfied by $\tau_1 := \tup[x/1]$; that is, $S_0$ is the set containing every $t$-node labeled by $0$ or $\neg x$, and $S_1$ is the set containing every $t$-node labeled by $1$ or $x$.  We hence have two $t$-nodes $v_{S_0}^t$ and $v_{S_1}^t$ labeled by $\neg x$ and $x$
  respectively. Observe that it is compatible with the determinism conditions on inputs. 
	
	We now proceed by induction to show how to build the gate $v_S^t$ for an internal node $t$ of $T$ and $S \subseteq N_t$.  We define $E(v_S^t) := \{(v_{S_1}^{t_1}, v_{S_2}^{t_2}) \mid (S_1,S_2) \text{ generates } S\}$. By induction, every model of $v_S^t$ is of the form $\tau_1 \times \tau_2$ where $\tau_1$ has shape $S_1$, $\tau_2$ has shape $S_2$ and $(S_1,S_2)$ generates $S$. Hence by \cref{lem:prop:tshapes-complete}, $\tau_1 \times \tau_2$ has shape $S$. Moreover, if $\tau$ has shape $S$ then by \cref{lem:prop:tshapes-complete} again, $\tau = \tau_1 \times \tau_2$ where there exists some shapes $(S_1,S_2)$ generating $S$ and $\tau_1$ has shape $S_1$, $\tau_2$ has shape $S_2$. In other words, the models of $v_S^t$ are exactly the set of assignments of $2^{X_t}$ of shape $S$, that is, $v_S^t$ computes $f_S^t$.
	
	We now argue that we can actually build $E(v_S^t)$ efficiently for every $t$-shape $S$. To do so, we can enumerate every $t_1$-shape $S_1$ and every $t_2$-shape $S_2$. This can be done because we have computed them inductively. Now, we can compute the $t$-shape  $S$ generated by $S_1,S_2$ using \cref{lem:prop:tshapes-compute} by going over every node $g$ in $N_t$ and checking whether $E(g)$ contains a pair $(g_1,g_2)$ with $g_1 \in S_1, g_2 \in S_2$. This can be done in time $O(k^3)$ because $N_t$ contains at most $k$ node. Moreover, each $t$-node contains at most $k^2$ pairs since $N_{t_1}$ and $N_{t_2}$ also have size bounded by $k$. Now, we simply add $(v_{S_1}^{t_1}, v_{S_2}^{t_2})$ as an input of $v_S^t$.
	
	Since there are at most $2^k$ possible $S_1$ and $2^k$ possible $S_2$, the construction of every $t$-node of $C'$ takes time $O(k^32^{2k}) = 2^{O(k)}$ and we have to compute it for every internal node $t$ of $T$. We hence have a total time of $|C| \times 2^{O(k)}$.
	
	We argue that $C'$ is deterministic. The inputs of $C'$ are clearly deterministic since by construction, we cannot have a $1$-labeled gate and literals, hence $C'$ respects this constraint. Now, let $t$ be an internal node of $T$ with children $t_1,t_2$. Let $S_1$ be a $t_1$-shape and $S_2$ be a $t_2$-shape and let $S$ be the unique shape they generate. Hence the pair $(v_{S_1}^{t_1}, v_{S_2}^{t_2})$ will only be used as an input of $v_S^t$ and no other. That is, $C'$ is deterministic.
	
	Finally, we argue that $C'$ is full. Indeed, let $\tau \in 2^{X_t}$. By definition, $\tau$ has a unique shape $S$, which may possibly be empty. By construction, we have a gate $v_S^t$, hence $v_S^t$ is satisfied by $\tau$. In other words, every $\tau \in 2^{X_t}$ satisfies exactly one $t$-gate of $C'$, that is, $C'$ is full. 
	
\end{proof}

It may be surprising that the TDD constructed in \cref{thm:prop:detTDD} is full. This is because $S=\emptyset$ is a $t$-shape for every node $t$ of the vtree. The shape $S=\emptyset$ catches every assignment in $2^{X_t}$ that does not satisfy any $t$-node of $C$, which is exactly what it means for the circuit to be complete. Interestingly, running the algorithm from~\cref{thm:prop:detTDD} on a TDD will produce a full TDD in time $O(\poly(k) \cdot |C|)$. Indeed, the only possible shapes for TDD are singletons or empty sets. The procedure will add the empty set shape whenever it is missing and plug it accordingly with what is below, which is exactly what the procedure from~\cref{thm:prop:tdd-neg} does. 
\end{toappendix}

\section{Bottom up compilation}
\label{sec:prop:struct-cnf}

The tractability ($\wedge$BC) for TDDs gives a natural algorithm for compiling a CNF formula into a
TDD, whose pseudo-code is given in \cref{alg:prop:tdd-bottomup}. The idea is to order the clauses of $F$ as $c_1,\dots,c_m$, build a TDD $T_i$ computing $c_i$ for every $i \leq m$ and then, iteratively construct a TDD $C_i$ computing $F_i := c_1 \wedge \dots \wedge c_i$ by observing that $F_i = F_{i-1} \wedge c_i$, using the algorithm for bounded conjunction. In the worst case, we could have $|C_i| = |c_i| \cdot |C_{i-1}|$, leading to an exponential blow-up in the size of the circuit. To avoid this if possible, we minimize the circuit after each conjunction. The only missing piece here is the fact that we can efficiently construct a TDD given a clause $c$. This can be done with a TDD of width $2$. For every node $t$ of the vtree $T$, we have two $t$-nodes. One computes $c_t := c|_{\var{t}}$ and the other computes $d_t := (\neg c)|_{\var{t}}$. The circuit is constructed by induction using the fact that, if $t$ has children $t_1,t_2$ then $d_t = d_{t_1} \wedge d_{t_2}$ and $c_t = (c_{t_1} \wedge c_{t_2}) \vee (c_{t_1} \wedge d_{t_2}) \vee (d_{t_1} \wedge c_{t_2})$. 



This kind of algorithm is usually referred to as ``bottom-up compilation'' and has been used for OBDD~\cite{cudd} and SDD~\cite{Choi_Darwiche_2013}. In this section, we investigate the complexity of \cref{alg:prop:tdd-bottomup} depending on the structure of the input CNF formula. We recover in a clean and modular way the fact that CNF formulas having bounded primal or incidence treewidth have TDDs of FPT size~\cite{BovaCMS15} and are able to easily generalize to bounded treewidth circuits~\cite{BovaS17,AmarilliCMS20}. 

\begin{algorithm}[tp]
  \caption{Bottom-up compilation into TDD.}
  \textbf{Input:} A CNF formula $F = c_1 \wedge \dots \wedge c_m$, a vtree $T$ over $\var{F}$.\\
  \textbf{Output:} A TDD computing $F$ respecting $T$.
  \begin{algorithmic}[1]
    \Procedure{CNF-to-TDD}{$F$, $T$}
    \State $C \gets $ TDD computing $1$ respecting $T$
    \For{$i = 1$ to $m$}
    \State $D \gets $ TDD computing $c_i$
    \State $C \gets$ construct a TDD for $C \wedge D$
    \State $\mathsf{minimize}(C)$
    \EndFor\\
    \Return{$C$}
    \EndProcedure
  \end{algorithmic}
  \label{alg:prop:tdd-bottomup}
\end{algorithm}

To this end, we use the notion of \emph{factor width}~\cite{BovaS17}. Given a Boolean function $f$ and a vtree $T$ over $X$, \cref{thm:prop:tdd-lowerbound,thm:prop:tdd-min} allow us to prove that the size of the smallest TDD for $f$ respecting $T$ is equal to $\sum_t S_t$ where the sum is over every node $t$ of $T$ and $S_t$ is the number of non-trivial $X_t$-subfunctions of $f$. Hence, the \emph{factor width of $f$ with respect to $T$}, \fw{f}[T] for short, is defined as $\max_t S_t$ where the maximum is over all nodes $t$ of $T$~\cite{BovaS17}. From what precedes, we know that the smallest 
TDD computing $f$ and respecting $T$ has width $\fw{f}[T]$ and size $O(|X| \cdot \fw{f}[T])$ since $T$ has at most $2|X|-1$ nodes. Factor width thus provides a good proxy for the size of the smallest TDD computing $f$. We also define the \emph{factor width of $f$} to be $\min_T \fw{f}[T]$, where $T$ goes over every vtree over the variables $X$ (the definition of factor width from~\cite{BovaS17} allows vtrees over variables $Z \supseteq X$ but these extra variables do not change the value of $\fw{f}$). We slightly abuse notation and for a given CNF formula $F$ and a vtree $T$ over $\var{F}$, write $\fw{F}[T]$ to denote the factor width of the Boolean function represented by $F$ \wrt{} $T$. We can then bound the runtime of \cref{alg:prop:tdd-bottomup} as follows:
\begin{theorem}[$\star$]
  \label{thm:prop:tdd-bottomup-complexity} Given a CNF formula $F$, a vtree $T$ over variables $X$ and an order $c_1,\dots,c_m$ on the clauses of $F$, \cref{alg:prop:tdd-bottomup} runs in time $m \cdot |X| \cdot \poly(k)$ where $k = \max_{i=1}^m \fw{c_1 \wedge \dots \wedge c_i}[T]$.
\end{theorem}
\begin{toappendix}
\begin{proof}[Proof of \cref{thm:prop:tdd-bottomup-complexity}]
At each iteration of the loop, constructing $D$ can be done in $O(|X|)$. Now, by \cref{prop:prop:tdd-and,thm:prop:tdd-minization-complexity}, both the apply and minimization run in $|X| \poly(k)$. Indeed, before the apply, $C$ is the minimal TDD computing $c_1\wedge \dots \wedge c_i$ which has width at most $k$ by \cref{thm:prop:tdd-min}. Since $D$ has width $2$, after the apply, $C$ has width $2k$. Hence, the minimization runs in $\poly(k)$.  The total complexity of \cref{alg:prop:tdd-bottomup} is then $m\cdot|X|\cdot \poly(k)$.
\end{proof}
\end{toappendix}
The proof of \cref{thm:prop:tdd-bottomup-complexity} is based on the fact that minimizing a TDD $C$ can be done in time $|X| \cdot \poly(w)$ where $w$ is the width of $C$. Similarly, computing a TDD for $C \wedge D$ can be done in time $|X| \cdot \poly(w)$.  The result follows from the fact that the width of every intermediate circuit built in the main loop of \cref{alg:prop:tdd-bottomup} will never exceed $k$. We now explore a few applications of \cref{thm:prop:tdd-bottomup-complexity}.

\textbf{Bounded primal and incidence treewidth.} CNF formulas of bounded primal or incidence treewidth have long been known to be tractable. It has long been known that SAT can be solved in time $2^{O(k)} \size{F}$. The earliest reference of this fact seems to be in a paper by Dantsin from 1979~\cite{dantsin1979parameters}, though it is not specifically stated with the treewidth terminology, later improved by Alekhnovich and Razborov~\cite{AlekhnovichR11, alekhnovich2002satisfiability}, where the result is expressed in terms of the equivalent branch-width measure, and Szeider~\cite{Szeider04}. The generalization for the tractability of \#SAT has first been observed by Sang, Bacchus, Beame, Kautz, and Pitassi in~\cite{sang04} and later generalized to the case of incidence treewidth by Samer and Szeider~\cite{SamerS10}. The existence of small d-DNNFs for such formulas is implicit in Darwiche's early contribution~\cite{Darwiche04} and explicit in a collaboration with Pipatsrisawat~\cite{PipatsrisawatD10} for primal treewidth. The case for incidence treewidth has been formally proven along more general results in~\cite{BovaCMS15}.
We revisit these results by showing that such formulas have small factor width. More precisely:

\begin{theorem}[$\star$]
  \label{thm:prop:ptw-vs-subw} Given a CNF formula $F$ and a tree decomposition $\calT$ of $\gprim$ of width $k$, we can construct a vtree $T$ over $\var{F}$ such that for every $F' \subseteq F$, $\fw{F'}[T] \leq 2^{k}$. 
\end{theorem}
\begin{toappendix}
\begin{proof}[Proof of \cref{thm:prop:ptw-vs-subw}]
  Let $\calT$ be a tree decomposition of $\gprim$ of width $k$. Before describing how to build a vtree, we give an intuition on why $F$ has small factor width. Let $t$ be a node of $T$ and let $B_{\leq t}$ be the set of variables appearing in a bag below $t$ in $\calT$. Let $c$ be a clause of $F$ such that $\var{c} \cap B_{\leq t} \neq \emptyset$. We claim that in this case, either $\var{c} \subseteq B_{\leq t}$ or $\var{c} \cap B_{\leq t} \subseteq B_t$. Indeed, assume that $\var{c} \not \subseteq B_{\leq t}$. That is, $c$ has a variable $y \notin B_{\leq t}$. Let $x \in B_{\leq t} \cap \var{c}$. The edge $\{x,y\}$ is an edge of $\gprim$ and is hence covered in some bag of $\calT$. Since $y \notin B_{\leq t}$, the bag covering $\{x,y\}$ is not below $t$. Hence, by connectivity, $x \in B_t$. This proves that $\var{c} \cap B_{\leq t} \subseteq B_t$.

  Now let $\tau \in 2^{B_{\leq t}}$. We claim that $F[\tau]$ only depends on the value of $\tau$ over $B_t$. Indeed, either $F[\tau]$ contains the empty clause and hence it computes the $0$-constant function. Otherwise, the clauses in $F$ not satisfied by $\tau$ are either clauses of $F$ with no variable in $B_{\leq t}$ or clauses of $F$ having at least one variable in $B_{\leq t}$ and one variable outside $B_{\leq t}$ (clauses having all their variables in $B_{\leq t}$ are satisfied by $\tau$ because $F[\tau]$ does not contain the empty clause). In other words, if $\tau$ and $\tau'$ are such that neither $F[\tau]$ nor $F[\tau']$ contains the empty clause and $\tau_{|B_t} = \tau'_{|B_t}$, then $F[\tau] = F[\tau']$. Hence they define the same subfunction. Since $|B_t| \leq k$, there are at most $2^k$ such subfunctions.

      We now explain how to construct the vtree $T$ from the lemma statement. We build $T$ such that for each internal node $t'$ of $T$, $X_{t'}$ is of the form $B_{\leq t} \setminus U$ for some node $t$ of $\calT$ and $U \subseteq B_t$.  From what precedes, the number of $X_{t'}$-subfunctions of $F$ will be at most $2^{k}$, since they will only depend on the value of an assignment over variables $B_t$. 

      To do so, observe that for each variable $x$ of $F$, we can find a unique node $t(x)$ in $T$ such that $x \in B_{t(x)}$ and $x$ is not in $X \setminus B_{\leq t(x)}$. In other words, $t(x)$ is the shallowest node of $T$ in which $x$ appears in. We build $T'$ by starting from $\calT$ and attaching a new leaf labeled by $x$ on the edge between $t(x)$ and its parent. We then make the tree binary by replacing nodes with $d$ children with binary tree having $d$ leaves and contract nodes with one child with their child. It is easy to see by construction that $T$ has the required property and that the factor width of $F$ wrt $T$ is at most $2^k$.

      Finally, observe that if $F' \subseteq F$, then $\calT$ is still a tree decomposition of $\gprim[F']$. Now since the vtree $T$ previously constructed only depends on $\calT$ and not on $F$, we have, by the same reasoning as above, that the factor width of $F'$ with respect to $T$ is at most $2^k$. 
\end{proof}
\end{toappendix}

\begin{theorem}[$\star$]
  \label{thm:prop:itw-vs-subw} Given a CNF formula $F$ and a tree decomposition $\calT$ of $\ginc$ of width $k$, we can construct a vtree $T$ over $\var{F}$ such that for every $F' \subseteq F$, $\fw{F'}[T] \leq 2^{k}$. 
\end{theorem}
\begin{toappendix}
\begin{proof}[Proof of \cref{thm:prop:itw-vs-subw}]
  The proof is similar. Consider a tree decomposition $T$ of $\ginc$ and a node $t$ of $T$. Let $V_{\leq t} = B_{\leq t} \cap X$, the variables appearing below $t$. Let $c$ be a clause of $F$ that has variables both in $V_{\leq t}$ and in $X \setminus V_{\leq t}$. We claim that, in this case, $c \in B_t$ or $\var{c} \cap V_{\leq t} \subseteq B_t$. Let $x \in \var{c} \cap V_{\leq t}$. By definition, $c$ contains a variable $y \in X \setminus V_{\leq t}$. Now, both edges $\{y,c\}$ and $\{x,c\}$ of $\ginc$ have to be covered in some bag of $T$. The bag covering $\{y,c\}$ cannot be below $t$ because $y$ is not in $V_{\leq t}$. Hence $c$ appears somewhere not below $t$. If it appears below $t$, then $c \in B_t$ by connectedness. Otherwise, $c$ does not appear below $t$ and hence, the edge $\{x,c\}$ is not covered below $t$. But since $x$ appears below $t$, by connectedness, it has to appear in $B_t$ too.   

  Now, fix $\tau \in 2^{V_{\leq t}}$, and consider $F[\tau]$. If it contains the empty clause, then it is the $0$-constant function. Otherwise, the remaining clauses only depend on the value of $\tau$ over $B_t$ and can possibly contain any subset of clauses appearing in $B_t$. Hence, there are at most $2^k$ possible CNF formulas for $F[\tau]$. In other words, $F$ has at most $2^k$ $V_{\leq t}$-subfunctions.
  
  The construction of the vtree $T'$ is similar to the one in \cref{thm:prop:ptw-vs-subw} to ensure that for every node $t'$ of $T'$, $X_{t'}$ is of the form $V_{\leq t} \setminus U$ for some $U \subseteq B_t$, which concludes the proof. 
\end{proof}
\end{toappendix}

We give an intuition on the proof of \cref{thm:prop:ptw-vs-subw}. For a node $t$ of $\calT$, let $Y_t$ be the set of variables of $F$ appearing in a bag below $t$. We claim that $F$ has at most $2^k$ $Y_t$-subfunctions. Indeed, assume that a clause $c$ of $F$ has variables both in $Y_t$ and in $X \setminus Y_t$. Then we must have  $\var{c} \cap Y_t \subseteq B_t$, where $B_t$ is the bag at node $t$ of $\calT$.
Let $\tau \in 2^{Y_t}$. Remove from $F[\tau]$ every clause already satisfied. Then $F[\tau]$ contains either clauses without variables in $Y_t$ and clauses having both in $Y_t$ and in $X \setminus Y_t$. From what precedes, $\var{c} \cap Y_t \subseteq B_t$, hence $F[\tau] = F[\tau|_{B_t}]$. Hence there are at most $2^{|B_t|} \leq 2^k$ $Y_t$-subfunctions. This proof works for every $F' \subseteq F$. It remains to build a vtree which induces roughly the same partitions as $Y_t$, which we explain in the appendix.

The case of incidence treewidth is very similar.  In this case however, a $Y_t$-subfunction induced by an assignment $\tau \in 2^{Y_t}$ is completely defined by the subset of clauses in $B_t$ that are satisfied by $\tau$ and by the value of $\tau$ in $B_t \cap X$. It still gives at most $2^k$ $Y_t$-subfunctions. 

The new connection established by \cref{thm:prop:ptw-vs-subw,thm:prop:itw-vs-subw} allows us to nicely recover the tractability results discussed before. Indeed, it is straightforward to see that if a CNF formula has primal or incidence treewidth $k$, then every sub-formula of $F$ has treewidth at most $k$. Hence, the bottom-up compilation to TDD from \cref{alg:prop:tdd-bottomup} runs in time $\poly(2^k) \cdot mn = 2^{O(k)} \cdot mn$ on a formula with $n$ variables, $m$ clauses and of primal or incidence treewidth $k$ by \cref{thm:prop:tdd-bottomup-complexity}, as long as we start from the vtree given by \cref{thm:prop:ptw-vs-subw,thm:prop:itw-vs-subw}.

\begin{theorem}
  \label{thm:prop:dtdd-tw}
Given a CNF formula $F$ of primal or incidence treewidth $k$, one can construct a TDD of width at most $2^k$ computing $F$ in time $2^{O(k)} \cdot mn$. 
\end{theorem}

\cref{alg:prop:tdd-bottomup} gives a conceptually simpler algorithm than the bottom-up dynamic programming on tree decomposition from~\cite{SamerS10} and serves as a nice example of the power of TDDs and minimization. 
The complexity of this approach is however not as good as earlier work where the dependency on the size of the CNF formula is linear. Maybe it can be fixed by minimizing the circuits while computing $C \wedge D$ in \cref{alg:prop:tdd-bottomup} and to have a dedicated algorithm to compute $C \wedge D$ in the case where $D$ represents a clause but we leave this question for further investigation.



\textbf{Circuit treewidth.}
%
Another interesting application of \cref{alg:prop:tdd-bottomup} is related to the compilation of bounded treewidth Boolean circuits, which can be seen as a generalization of incidence treewidth. The treewidth of a Boolean circuit $C$ is defined as the treewidth of its underlying graph. For the notion to make sense, one needs to first assume that for any variable $x \in X$, there is at most one input of the circuit labeled by $x$. From~\cite{BovaS17}, we know that the factor width of a Boolean circuit of treewidth $k$ is bounded by $2^{2^{O(k)}}$. We improve to $3^{k+2}$, getting a single exponential in $k$ and show that bottom-up compilation can be used to recover a result from~\cite{AmarilliCMS20} showing that bounded treewidth circuit can be compiled into structured d-DNNF of size $2^{O(k)}|C|$.

For a Boolean circuit, a \emph{subcircuit}  $C'$ of $C$ is a subset of nodes and edges of $C$ forming a valid Boolean circuit (that is, its inputs are labeled with variables). 
\begin{theorem}[$\star$]
  \label{thm:prop:ctw-sfw} Let $C$ be a Boolean circuit over variables $X$ and let $\calT$ be a tree decomposition of $C$ of treewidth $k$. We can construct a vtree $T$ such that for every subcircuit $C'$ of $C$ computing a function $f'$,  we have $\fw{f'}[T] \leq 3^{k+2}$.
\end{theorem}
\begin{toappendix}
\begin{proof}[Proof of \cref{thm:prop:ctw-sfw}]
  To explain our approach, we need a few definitions first. Given a partial assignment $\tau \in 2^Y$ for some $Y \subseteq X$ and a gate $v$ of $C$, we define \emph{the value of $v$ under $\tau$}, denoted by $val(v,\tau)$, as a value in $\{0,1,\bot\}$. Intuitively, $val(v,\tau)$ is the value of $v$ under the partial evaluation, if it is determined by it. Otherwise, $val(v,\tau) = \bot$, meaning that its value is still undefined. A gate may have inputs whose value is undefined but still have a fixed value: for example, it is enough for a $\wedge$-gate $v$ to have one input with value $0$ to force the value of $v$ to $0$. More formally, we define $val$ inductively as follows: if $v$ is an input gate labeled by $y \in Y$, then $val(v,\tau) = \tau(y)$. If $y \notin Y$, then $val(v,\tau) = \bot$. If $v$ is a negation gate with input $w$ and $val(w,\tau)=\bot$, then $val(v,\tau) = \bot$. Otherwise, $val(v,\tau) = 1 - val(w,\tau)$. If $v$ is an $\vee$-gate with input $w_1,\dots,w_k$ then: $val(v,\tau)=1$ if  $val(w_i,\tau) = 1$ for some $i \leq k$, $val(v,\tau) = 0$, if $val(w_i,\tau) = 0$ for every $i \leq k$, and $val(v,\tau)=\bot$ otherwise. If $v$ is a $\wedge$-gate with input $w_1,\dots,w_k$ then: $val(v,\tau)=0$ if  $val(w_i,\tau) = 0$ for some $i \leq k$, $val(v,\tau) = 1$, if $val(w_i,\tau) = 1$ for every $i \leq k$, and $val(v,\tau)=\bot$ otherwise.

We first construct a vtree $T$ such that  $\fw{C}[T] \leq 3^{k+2}$, and then explain how the same vtree works for every subcircuit $C'$.  Since the treewidth of $C$ is $k$, there must exist a tree decomposition of $C$ in which every bag has size at most $k+1$. Let $\calT$ be the tree decomposition that we get from an optimal tree decomposition by adding the output $o$ of $C$ to every bag. Obviously, $\calT$ is a tree decomposition of $C$ of width at most $k+1$. Given a node $t$ of $\calT$, we denote by $X_{\leq t}$ the set of variables $x$ such that the input gate labeled by $x$ of $C$ appears in a bag below $t$ in $\calT$. We claim that for every $t$, $f_C$ has at most $3^{k+2}$ distinct $X_{\leq t}$-subfunctions. Indeed, given $\tau \in 2^{X_{\le t}}$, let $b_\tau$ be the function mapping $v \in B_t$ to $val(v,\tau)$. We claim that for $\tau,\nu \in 2^{X_{\le t}}$, if $b_\tau = b_\nu$, then $\tau$ and $\nu$ define the same $X_{\leq t}$-subfunction. Since there are at most $|\{0,1,\bot\}|^{|B_t|} \leq 3^{k+2}$ possible functions from $B_t$ to $\{0,1,\bot\}$, this will be enough to prove the desired result.

  We prove that $b_\tau = b_\nu$ implies that $f_C[\tau] = f_C[\nu]$. Assume toward a contradiction that $f_C[\tau] \ne f_C[\nu]$. It follows that there exists some $\sigma \in 2^{X \setminus X_t}$ such that $f_C(\tau \times \sigma) \neq f_C(\nu \times \sigma)$. Recall that $o$ denotes the output of $C$ and that by construction, $o \in B_t$. Since we have $val(o,\tau) = val(o,\nu)$ by assumption, we must have $val(o,\tau) = val(o,\nu)=\bot$, otherwise, we would have $val(o,\tau \times \sigma) = val(o,\nu \times \sigma)$, that is $f_C(\tau \times \sigma) = f_C(\nu \times \sigma)$.

  Now, it follows that there is at least one $g$ of $C$ such that:
  \begin{enumerate}
  \item \label{it:ctw:inbag} $g \in B_t$, 
  \item \label{it:ctw:discr} $val(g,\tau \times \sigma) \neq val(g,\nu \times \sigma)$ (in particular $val(g,\tau)=val(g,\nu)=\bot$).
  \end{enumerate}
  
  Consider a deepest gate $v_0$ for which this is true, that is, $v_0 \in B_t$, $val(v_0,\tau \times \sigma) \neq val(v_0,\nu \times \sigma)$ and no gate below $v_0$ verifies these two properties. Observe that $v_0$ cannot be an input of the circuit, otherwise it would be labeled with some variable $x \in X_{\leq t}$ and we would have $\tau(x) = \nu(x)$ since $v_0 \in B_t$, contradicting the fact that $val(v_0,\tau \times \sigma) \neq val(v_0,\nu \times \sigma)$. We claim that, since $val(v_0,\tau \times \sigma) \neq val(v_0,\nu \times \sigma)$, $v_0$ must have at least one input $v_1$ such that $val(v_1,\tau \times \sigma) \neq val(v_1,\nu \times \sigma)$ and $val(v_1,\tau) = \bot$ or $val(v_1,\nu)=\bot$. Indeed, if for every input $v$ of $v_0$ we have $val(v,\tau \times \sigma) = val(v,\nu \times \sigma)$ then this would contradict the fact that $v_0$ verifies \cref{it:ctw:discr}. To see that $val(v_1,\tau) = \bot$ or $val(v_1,\nu)=\bot$, let $v$ be an input of $v_0$ such that $val(v,\tau \times \sigma) \neq val(v,\nu \times \sigma)$ and assume $val(v,\tau) \in \{0,1\}$ and $val(v,\nu) \in \{0,1\}$. In this case, we claim that we cannot have $val(v,\tau)=val(v,\nu)$: W.l.o.g., assume $val(v,\tau)=0$ and $val(v,\nu)=1$.
  Now, if  $v_0$ is a $\vee$-gate, then $val(v_0,\nu) = 1$, if $v_0$ is a $\wedge$-gate, then $val(v_0,\tau) = 0$ and if $v_0$ is a $\neg$-gate, $val(v_0,\tau)=1$. In any case, it contradicts that $val(v_0,\nu)=val(v_0,\tau)=\bot$.

  To sum up, we have shown that there exists an input $v_1$ of $v_0$ such that $val(v_1,\tau \times \sigma) \neq val(v_1,\nu \times \sigma)$ and either $val(v_1,\tau) = \bot$ or $val(v_1,\nu) = \bot$. W.l.o.g., we assume that $val(v_1,\tau)=\bot$. We now construct two paths from $v_1$ to an input of $C$:

  \begin{itemize}
  \item The first one is built on the following observation: $v_1$ must have at least one input $v_2$ such that $val(v_2,\tau \times \sigma) \neq val(v_2, \nu \times \sigma)$, otherwise, we would have $val(v_1,\tau \times \sigma) = val(v_1, \nu \times \sigma)$. By iterating, we build a path $P_1 = (v_1,\dots,v_p)$ to some input $v_p$ of the circuit such that for every $i \leq p$, $val(v_i,\tau \times \sigma) \neq val(v_i, \nu \times \sigma)$. In particular, $v_p$ is labeled with some variable $x \in X_{\leq t}$ such that $\nu(x) \neq \tau(x)$.
  \item The second one is built on the following observation: $v_1$ must have at least one input $w_2$ such that $val(w_2,\tau)=\bot$. Otherwise, $val(v_1,\tau)$ would take a value in $\{0,1\}$. By iterating on this observation, we build a path $P_2 = (v_1,w_2,\dots,w_q)$ to some input $w_q$ of the circuit such that for every $i \leq q$, $val(w_i,\tau)=\bot$. In particular, $w_q$ is labeled with some variable $y \in X \setminus X_{\leq t}$, since it cannot be assigned by $\tau$.
  \end{itemize}

  Finally, consider the path $P$ from $v_p$ to $w_q$ obtained by concatenating $P_1$ with the reverse of $P_2$. It only contains gates below $v_1$ by definition, hence below $v_0$ in particular. Moreover, $v_p$ being labeled by some $x \in X_{\leq t}$ with $\tau(x) \neq \nu(x)$, it must appear strictly below $t$ in $\calT$. And since $w_q$ is labeled by some $y \notin X_{\leq t}$, it does not appear below $t$. Hence $P$ must contain a gate $g$ in $B_t$ by the connectedness of $\calT$. Since $val(g,\tau)=val(g,\nu)$ by assumption, if $g$ is on path $P_1$, then $val(g,\tau)=val(g,\nu) = \bot$, otherwise, we cannot have $val(g,\tau \times \sigma) \neq val(g,\nu \times \sigma)$. If $g$ is on path $P_2$, then by construction, we have $val(g,\tau)=\bot$. And since $g \in B_t$, we must also have $val(g,\nu)=\bot$. Recall that $g$ appears in the subcircuit rooted in $v_0$ by construction. But then, it contradicts the fact that $v_0$ is a deepest gate satisfying \cref{it:ctw:inbag,it:ctw:discr} and we must have $f_C[\tau]=f_C[\nu]$.

  We obtain the corresponding vtree $T$ for $f_C$ from $\calT$ inducing the same partitions as $(X_{\leq t}, X \setminus X_{\leq t})$ following the same construction as \cref{thm:prop:ptw-vs-subw}. Again, the construction of $T$ from $\calT$ does not depend on the function but only on $\calT$ and $\calT$ is also a tree decomposition of $C'$, a subcircuit of $C$.  Hence, we have $\fw{f'}[T] \leq 3^{k+2}$. 
\end{proof}
\end{toappendix}

It gives a straightforward way of constructing a TDD of size $2^{O(k)}|C|$ computing $f_C$ from a tree decomposition $\calT$ of treewidth $k$ of a Boolean circuit $C$. We first extract a vtree $T$ as in \cref{thm:prop:ctw-sfw} and, for every gate $g$ of $C$, we build a TDD $T_g$ computing the same Boolean function as $g$. For example, for a $\wedge$-gate $g$ of $C$ with input $g_1,\dots, g_p$, construct $T_g$ as follows: inductively construct $T_{g_1}, \dots, T_{g_p}$, then iteratively build $((T_{g_1} \wedge T_{g_2}) \wedge T_{g_3}) \wedge \dots \wedge T_{g_p}$. For every $i \leq p$, the circuit having $g$ as output and $g_1,\dots, g_i$ as input is a subcircuit of $C$, hence the resulting intermediate TDD have width at most $3^{k+2}$. We proceed similarly for $\neg$-gates and $\vee$-gates.  Since computing an optimal tree decomposition can be done in FPT linear time~\cite{Bodlaender93a}, we have:
\begin{theorem}
\label{thm:tdd:circuit-treewidth}  Given a Boolean circuit $C$ of treewidth $k$, we can compute a vtree $T$ and a TDD respecting $T$ computing $f_C$ of width $2^{O(k)}$ and in time $2^{O(k)} \cdot |X| \cdot |C|$.
\end{theorem}

Treewidth is not the most general parameter for which we can build polynomial-size d-DNNF. CNF formulas of bounded MIM-width, for example, may have unbounded treewidth but have polynomial-size deterministic DNNF circuits~\cite{BovaCMS15,SaetherTV14}. That said, it is not clear whether they have bounded factor width. Such a result would allow to show that \cref{alg:prop:tdd-bottomup} works in polynomial time on bounded MIM-width instances, simplifying the convoluted algorithm from the literature. We leave the study of such graph measures for future work.

\section{Comparing TDD with other data structures}
\label{sec:prop:tdd-vs-rest}

\textbf{OBDD.} Tractable queries and transformations for TDD are similar to those for OBDD. We can actually see OBDD as a particular subclass of TDD where the underlying vtree $T$ is \emph{linear}, that is, for every internal node $t$ of $T$, one child of $t$ is a leaf. A linear vtree $T$ over variables $X$ naturally induces an order $\pi_T = (x_1,\dots,x_n)$ on $X$ defined as follows: $x_1$ is the leaf attached to the root of $T$, and $(x_2,\dots,x_n)$ is the order induced by the other subtree attached to the root. Similarly, an order $\pi = (x_1,\dots,x_n)$ can be mapped naturally to the linear vtree $T_\pi$ defined as the vtree whose root has one leaf child labeled by  $x_1$, and its other child is the vtree for the order $(x_2,\dots,x_n)$. 
For an order $\pi=(x_1,\dots,x_n)$, we let $\pi^{-1}=(x_n,\dots,x_1)$. The class of TDD with linear vtrees corresponds exactly to OBDD in the following sense:

\begin{theorem}[$\star$]
  \label{thm:obdd-tdd-linear}
  Given an OBDD $C$ over variables $X$ and order $\pi = (x_1,\dots,x_n)$, one can construct an equivalent
  TDD respecting $T_{\pi^{-1}}$ of size at most $3|C|$ in time $O(|C|)$.
  Similarly, let $T$ be a linear vtree and $C$ be a TDD respecting $T$. Then one can construct an equivalent OBDD in time $O(|C|)$ respecting order $\pi_T^{-1}$. 
\end{theorem}
\begin{toappendix}
\begin{proof}[Proof of \cref{thm:obdd-tdd-linear}.]
We start by explaining how to convert an OBDD to a TDD. Intuitively, the TDD corresponds to rooting the OBDD in its $1$-sink. We assume wlog that $C$ has exactly two sinks: a $1$-labeled sink and a $0$-labeled sink. For each decision node $v$ other than the source of $C$, we introduce a node $g_v$ in the TDD. The models of $g_v$ are exactly the set of assignments that are compatible with a path from the source of $C$ to $v$.
First assume $v$ is a decision-node on $x_2$. Either every assignment of variable $x_1$ leads to $v$ (that is, both outgoing edges of the source of $C$ are connected to $v$), in which case we define $g_v$ to be the unique $\ell_1$-node labeled by $\top$. Otherwise, we define $g_v$ to be the $\ell_1$ node labeled by the literal corresponding to the path leading to $v$. More precisely:
\begin{itemize}
  \item If both outgoing edges of the source of $C$ are connected to $v$, then $g_v$ is labeled by $1$ (and is the only $\ell_1$-node of $C'$)
\item If only the $1$-labeled outgoing edge of the source of $C$ is connected to $v$ then $g_v$ is labeled by $x_1$. 
\item If only the $0$-labeled outgoing edge of the source of $C$ is connected to $v$ then $g_v$ is labeled by $\neg x_1$.    
\end{itemize}
	
By definition, the models of $g_v$ are clearly the set of assignments of $\{x_1\}$ whose path in $C$ ends in $v$. Now assume $v$ is a decision-node on $x_{i+1}$ ($i>1$) and that $g_w$ has been constructed for every node $w$ before $v$ in $C$. We let $g_v$ be a new $p_i$-node and we introduce two $\ell_i$ nodes $L^1_i$ and $L_i^0$ respectively labeled by $x_i$ and $\neg x_i$. The inputs of $g_v$ are defined as follows: it contains a pair $(L^1_i, g_w)$ for every node $w$ such that there is an edge labeled by $1$ from $w$ to $v$ and a pair $(L_i^0, g_w)$ for every node $w$ such that there is an edge labeled by $0$ from $w$ to $v$. By induction, the models of $g_v$ are the assignments $\tau \in 2^{\{x_1,\dots,x_i\}}$ such that the path of $\tau$ goes to $w$ and such that, $\tau(x_i) = b$ if and only if there is a $b$-labeled edge between $w$ and $v$. But this is exactly what it means for the path of $\tau$ to end in $v$. Hence the induction hypothesis is preserved.
	
If $v$ is a sink, we do the same construction as above but over variable $x_n$. That is, we have two $\ell_n$-nodes $L_n^0$ and $L_n^1$ labeled by $\neg x_n$ and $x_n$ respectively. We introduce a $p_n$-node $g_v$ whose inputs are pairs $(L_i^b,g_w)$ if $(w,v)$ is an edge of $C$ labeled by $b \in \{0,1\}$. As before the models of $g_v$ are exactly the assignments whose path ends in $v$. Hence, if we choose the output of the $C'$ to be $g_v$ where $v$ is the $1$-sink of $C$, the models of $C'$ are the same as the models of $C$. Observe that we have not changed the width of the circuit.  

We now explain how to convert a TDD into an OBDD. We only sketch the construction. We have one decision-gate $v(g)$ in $C'$ for each $t$-node $g$ where $t$ is an internal node of $T$. We define it as follows:  Let $g$ be a $t$-node of $C$ where $t$ is an internal node of $T$. Let $u$ be the parent of $t$ in $T$ and $\ell$ be the sibling of $t$ in $T$. By definition, $\ell$ is labeled by some variable $x_i$ since $T$ is linear. We add a decision-gate $v(g)$ in the OBDD over variable $x_i$. The $1$-edge of $v(g)$ is connected to $v(h)$ where $h$ is the only $u$-node containing input $(x_i, g)$ or $(\top,g)$. If no such $h$ exists, we plug the $1$-edge of $v(g)$ to a $\bot$-sink. We connect the $0$-edge of $v(g)$ similarly to $v(h)$, where $h$ is the only $u$-node containing input $(\neg x_i, g)$ or $(\top, g)$. 
\end{proof}
\end{toappendix}

The proof of \cref{thm:obdd-tdd-linear} mainly boils down to rooting an OBDD in its $1$-sink, as illustrated in \cref{fig:prop:tdd-vs-obdd}. 
We observe however that \cref{thm:prop:tdd-lowerbound} offers a way of getting the correspondence of \cref{thm:obdd-tdd-linear} in a non-constructive way. Indeed, it is known~\cite{hayase1998obdds} that the width of the minimal OBDD is exactly the maximum number of subfunctions one can get by fixing variables $x_1,\dots,x_i$ for some $i \leq n$. This exactly corresponds to the number of subfunctions we can have with a linear vtree. The previous constructions can be extended to the case of non-deterministic TDD and non-deterministic OBDD.

\begin{figure}
	\centering
	\includegraphics[page=1]{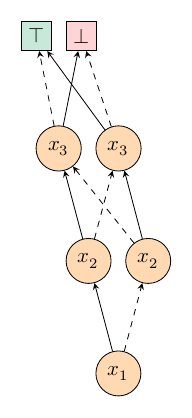}
  \qquad
	\includegraphics[page=2, width=6cm]{figs/tdd-parity.pdf}
	\caption{An OBDD (left) represented as a TDD (right). The output of the TDD is represented in green and each node $t$ of the vtree is made explicit by a rectangle around $t$-nodes.
        }
	\label{fig:prop:tdd-vs-obdd}
\end{figure}

OBDDs are however weaker than TDDs. This follows as a corollary of~\cite{Razgon14} and \cref{thm:prop:dtdd-tw}. In~\cite{Razgon14}, Razgon proves that there is  a family of CNF formulas $(F_n)_{n \in \N}$ where $F_n$ has $n$ variables and treewidth $O(\log n)$ such that every OBDD representing $F_n$ must have size $n^{\Omega(\log n)}$, while \cref{thm:prop:dtdd-tw} shows that such instances have polynomial-size TDD representation.

\begin{theorem}
	\label{thm:prop:sep-obdd-tdd}
	There exists a family $(F_n)_{n \in \N}$ of CNF formulas such that $F_n$ can be represented by a polynomial-size TDD while every OBDD representing $F_n$ has size at least $n^{c \log n}$ for some constant $c$. 
\end{theorem}

The separation given by \cref{thm:prop:sep-obdd-tdd} is only quasi-polynomial, and one can wonder whether a truly exponential separation is possible. It may seem possible that TDDs can be quasi-polynomially simulated by OBDDs, in the same way FBDDs quasi-polynomially simulate decision-DNNF circuits~\cite{BeameLRS13}. We leave this question for future investigation.


\textbf{Deterministic DNNF.} As we have already observed, TDD is a subclass of structured deterministic DNNF. We show in this section that structured deterministic DNNF may be exponentially smaller than TDD. To do so, we are interested in the \emph{Hidden Weighted Bit} functions, which are known to be hard for OBDD~\cite{Wegener00}. Given $n \in \N$, define $ \HWB_n(x_1,\dots,x_n) = 1$ if and only if the value assigned to $x_S$ is $1$, where $S = \sum_{i=1}^n x_i$. The Boolean function $\HWB_n$ can easily be computed by a structured d-DNNF, by first guessing the number $S$ of variables set to $1$ and then checking that this is indeed the case and that $x_S=1$ with a small OBDD. However, $\HWB_n$ does not admit polynomial-size OBDD~\cite{Wegener00}. We adapt the $\HWB_n$ lower bound for OBDD to TDD. The proof relies on adapting \cite[Lemma 4.10.1]{Wegener00}. In a nutshell, this lemma shows that if we pick $Y \subseteq \{x_1,\dots,x_n\}$ of size $n/2$, then $\HWB_n$ has an exponential number of $Y$-subfunctions. We generalize it to show that if $Y$ has size between $n/3$ and $2n/3$ then we still have an exponential number of $Y$-subfunctions. We then apply the lemma by finding a node $t$ in the vtree such that $X_t$ has size between $n/3$ and $2n/3$ which gives an exponential lower bound on the size of a TDD computing $\HWB_n$.

\begin{toappendix}
\begin{lemma}[($\star$) Adapted from {\cite[Lemma 4.10.1]{Wegener00}}.]
	\label{lem:prop:segment-subf}
	Let $n \in \N$ and $I \subseteq [n]$. Assume there exist $p, \ell \in [n]$ such that:
		 $J := I \cap [p; p+\ell]$ has size $t$,
		 $\ell \leq n - |I|$, and
		 $n \geq p+\ell+t/2$.
	Then $\HWB_n$ has at least $t \choose t/2$ subfunctions over $X_I := \{x_i \mid i \in I\}$.
\end{lemma}
  \begin{proof}[Proof of \cref{lem:prop:segment-subf}]
	Let $A$ be the set of assignments of $X_J$ setting exactly $t/2$ variables to $1$ (and $t/2$ to $0$). Given $a \in A$, we define $a^* \in 2^{X_I}$  as follows: we set the largest possible number of variables $X_I \setminus X_J$ to $1$ so that the total number of variables set to $1$ by $a^*$ does not exceed $p$. Since $|X_I \setminus X_J| = |I|-t$ and $a$ already sets $t/2$ variables to $1$, we have two cases:
	\begin{itemize}
		\item either $|I|-t \geq p-t/2$, in which case, we pick $p-t/2$ variables from $X_I \setminus X_J$, set them to $1$ and then the other to $0$. In this case, we have that $a^*$ sets exactly $p$ variables to $1$.
		\item or $|I|-t < p-t/2$, in which case, we set every variables from $X_I \setminus X_J$ to $1$, in which case, $a^*$ sets $t/2+|I|-t = |I|-t/2<p$ variables to $1$. 
	\end{itemize}
	Observe that the number of variables set to $1$ by $a^*$ does not depend on $a$, but only on whether $|I|-t/2 \geq p$. We denote this number by $N_1$. 
	
	We claim that given distinct $a,b \in A$, $\HWB_n[a^*]$ and $\HWB_n[b^*]$ are two different subfunctions of $\HWB_n$ over $X_I$. Indeed, since $a$ and $b$ are distinct, there exists $i \in J$ such that $a(x_i) \neq b(x_i)$, and in particular, $a^*(x_i) \neq b^*(x_i)$. Assume wlog that $a^*(x_i) = 1$ and $b^*(x_i) = 0$. Now if we pick $c \in 2^{X \setminus X_I}$ such that $c$ sets exactly $i-N_1$ variables to $1$ (assume for now that this is possible). In this case, $a^* \times c$ and $b^* \times c$ both set exactly $i$ variables to $1$. But then $a^* \times c$ is a model of $\HWB_n$ since $a^*(x_i)=1$ but $b^* \times c$ is not since $b^*(x_i)=0$. Hence $\HWB_n[a^*]$ and $\HWB_n[b^*]$ are distinct subfunctions. In other words, $\HWB_n$ has at least $|A| = {t \choose t/2}$ subfunctions over $X_I$.
	
	It remains to show that we can pick $c \in 2^{X \setminus X_I}$ so that it sets $i-N_1$ variables to $1$. First, observe that $i-N_1 \geq 0$ since by definition, $N_1 \leq p$ and $i \in [p;p+\ell]$. We now have to prove that $|[n] \setminus I| = n-|I| \geq i-N_1$. Since $i \leq p + \ell$, we will prove $n-|I| \geq p+\ell-N_1$. We distinguish two cases. First assume $|I|-t/2 \geq p$, that is $N_1 = p$. In this case, we need to prove $n-|I| \geq \ell$ but this was part of our assumption. Otherwise, $|I|-t/2 < p$ and $N_1=|I|-t/2$. We hence have to prove $n-|I| \geq p+\ell-|I|+t/2$, that is $n \geq p+\ell+t/2$ which is also part of our assumption. 
\end{proof}

\end{toappendix}

\begin{theorem}[$\star$]
	\label{thm:prop:hwb-tdd}
	Let $n \in \N$ be a multiple of $7$. Then $\fw{\HWB_n} \geq 2^{cn}$ for some constant~$c$. In particular, any TDD computing $\HWB_n$ has size at least $2^{cn}$. 
\end{theorem}
\begin{toappendix}
\begin{proof}[Proof of \cref{thm:prop:hwb-tdd}]
	Let $T$ be a vtree over $\{x_1,\dots,x_n\}$ and let $t$ be a node of $T$ such that $n/3 \leq |X_t| \leq 2n/3$. We let $I = \{i \mid x_i \in X_t\} \subseteq [n]$. We show that we can find a sub-interval $J = [p;p+\ell]$ respecting the conditions of \cref{lem:prop:segment-subf} with $t = \beta n$ for some constant $\beta$. This is enough to prove the claim as ${\beta n \choose \beta n/2} \geq 2^{\beta n \over 2}$.
	
	We start by splitting $[n]$ into seven segments of size $n/7$: $J_0 = [1, n/7], \dots, J_6 = [6n/7, n]$. Since $J_0 \cup J_6$ has size at most $2n/7$, $I \setminus (J_0 \cup J_6)$ has size at least $|I| - |J_0| - |J_6| \geq (1/3-2/7)n = n/21$. It follows that there must be $J_i$ for $1 \leq i \leq 5$ such that $J_i \cap I$ has size $s \geq (n/21)/5 = n/105$. We let $p = i(n/7)+1$, $\ell = n/7$, $J = J_i = [p;p+\ell]$ and $s = |J \cap I| \geq n/105$.
	We have to verify the conditions of \cref{lem:prop:segment-subf}, namely:
	\begin{itemize}
		\item $\ell \leq n-|I|$. This follows as $\ell = n/7 \leq n/3 = n - 2n/3 \leq n-|I|$ since $|I| \leq 2n/3$.
		\item $n \geq p+\ell+s/2$. In our case, $p+\ell+s/2 = i(n/7)+1+n/7+n/210 \leq 6n/7+1+n/210 \leq n$.
	\end{itemize}
	Hence, by \cref{lem:prop:segment-subf}, $\HWB_n$ has at least $2^{s/2} \geq 2^{n/210}$ subfunctions over $X_t$.
\end{proof}
\end{toappendix}

\textbf{SDD.} TDDs and SDDs have a lot in common: they are both restrictions of structured deterministic DNNF with canonical representations, efficient negations and efficient apply. 
%
%
%

\begin{toappendix}
  \textbf{SDD.} We recall here the definition of SDDs and their restrictions. SDD~\cite{Darwiche11} is a restriction of structured d-DNNF and a natural generalization of OBDD based on the notion of $X$-partitions. An \emph{$X$-partition} is a set of Boolean functions $p_1,\dots,p_k$ over $X$ such that for every $\tau \in 2^X$, there exists a unique $i\leq k$ such that $\tau$ is a model of $p_i$.
Let $T$ be a vtree over variables $X$ and $t$ a node of $T$. An SDD $C$ that respects $t$ is either:
\begin{itemize}
	\item A single gate $g$ labeled by a constant $\{0,1\}$, computing the respective constant function,
	\item If $t$ is a leaf of $T$ labeled by $x$, it can be a single gate $g$ labeled by a literal $x$, $\neg x$. In this case, it computes the Boolean function associated to each literal.
	\item If $t$ is an internal gate of $T$ with children $t_1,t_2$ then $C$ is given by a gate $g$ labeled by $\{(p_1,s_1),\dots,(p_k,s_k)\}$ where for every $i \leq k$, $p_i$ is an SDD respecting $t_1$, $s_i$ is an SDD respecting $t_2$ and $\{p_1, \dots, p_k\}$ is a $X_{t_1}$-partition. In this case, the models of $g$ are the assignments of the form $\tau_i \times \sigma_i$ where $\tau_i$ is a model of $p_i$ and $\sigma_i$ a model of $s_i$.
\end{itemize}


An SDD is \emph{trimmed} if it has neither an internal gate of the form $\{(\top,\alpha), (\bot, \alpha)\}$ nor $\{(\top,\alpha)\}$.
An SDD is \emph{compressed} if every internal node has the property that $s_i \neq s_j$ for every $i < j \leq k$. One can transform an SDD to a compressed SDD by duplicating
gates that are used several times in every internal node. This transformation cannot be done in polynomial space: the duplication may cascade, resulting in an exponential blow up. Compressed and trimmed SDD are canonical: for every Boolean function $f$ over variables $X$ and vtree $T$, there is a unique, trimmed and compressed SDD respecting $T$ and computing $f$~\cite{Darwiche11}.

\end{toappendix}

Bova~\cite{Bova16} proved that $\HWB_n$ can be computed by an SDD of size $O(n^3)$. He also constructs a Boolean function $F(X,Y)$ such that $F(X,1,\dots,1) = \HWB_n(X)$ which has a compressed SDD of size $O(n^3)$. This establishes:

\begin{theorem}
Compressed SDD cannot be polynomially simulated by TDD.
\end{theorem}

The other way around, it turns out that we can always simulate a TDD by a polynomial-size SDD. If we allow encoding variables, since a TDD and its negation are both polynomial-size structured d-DNNFs, it follows by~\cite[Theorem 1]{BolligF21}. We show below that this is possible even without the encoding variables: 

\begin{theorem}[$\star$]
  \label{thm:tdd<sdd}
  Given a TDD $C$ respecting vtree $T$, one can construct a vtree $T'$ and an SDD $C'$ respecting $T'$ such that $C'$ computes the same function as $C$ and $|C'| = O(|C|^2)$. 
\end{theorem}
\begin{toappendix}

\begin{proof}[Proof of \cref{thm:tdd<sdd}.]
  We first assume that $C$ is full, that is, for every node $t$ of $T$, $\bigvee_{r} r = \top$, where $r$ runs over the set of $t$-nodes. This can be done by increasing the width of $C$ by at most $1$ by \cref{prop:prop:tdd-full}. Moreover, for simplicity, we assume that $C$ does not use constant gates: we remove $\bot$-gates as in \cref{thm:prop:zero-elim} and we replace $\top$-gates with two gates $x$, $\neg x$. For a node having an input of the form $(g, \top)$ for some $\top$-gate of $C$, we replace it with two inputs $(g,x)$ and $(g,\neg x)$.

  Let $r$ be a $t$-node and let $n$ be one of its descendants. By definition, $n$ is a $u$-node for some node $u$ of $T$ which is a descendant of $t$. Define $r|n$ to denote the Boolean function over variables $Y_t = X_t \setminus X_u$ whose models are the assignments $\tau \in 2^{Y_t}$ such that there exists a model $\sigma \in 2^{X_u}$ of $n$ such that $\tau \times \sigma$ is a model of $r$. We have:
\begin{lemma} \label{lem:decomp}
Let $n_1, \dots, n_k$ be the $u$-nodes that are descendants of $r$. Then, 
\[ r \equiv \bigvee_i n_i \land (r|n_i) . \]
\end{lemma}
\begin{proof}
  By definition, every model on the right-hand side of the equation is a model of $r$. We now show that every model of $r$ is of the form $n_i \land (r|n_i)$ for some descendant $n_i$ of $r$. The proof is by induction on the distance between $t$ and $u$. If $u$ is a child of $t$, it is clear because every model of $r$ is a model $\tau$ of $g_1 \wedge g_2$ where $g_2$ is a $u$-node. Otherwise, let $u'$ be the parent of $u$. We have by induction that every model of $r$ is a model of $n' \land (r|n')$ for some $u'$-node $n'$ which is a descendant of $g_2$, hence of $r$. Now the models of $n'$ are models of $m_1 \wedge n$ where $n$ is a $u$-node. Hence $\tau|_{X_u}$ is a model of $n$, which is a $u$-node and a descendant of $r$, which proves the induction. 
\end{proof}

First, we assume that the inner nodes of $T$ have a left child and a right child, chosen arbitrarily. We first construct a DNNF $C'$ computing the same function as $C$ such that every $\vee$-node of $C'$ is a \emph{partition node}, that is, a node of the form $\bigvee_{i=1}^k p_i \wedge s_i$ where $p_i$ and $p_j$ have disjoint models for every $i < j \leq k$ and $\bigvee_{i=1}^k p_i = \top$. 

Circuit $C'$ contains the following gates:
\begin{itemize}
\item a gate $g_{r|n}$, computing $r|n$ for every node $r$ of $C$ and $n$ where $n$ is a right descendant of $r$,
\item a gate $h_r$ computing $r$ for every node $r$ of $C$
\end{itemize}
We construct gates $g_{r|n}$ and $h_r$ by induction on the distance between $n$ and $r$ and bottom up induction on~$r$.

\underline{Base case:} $\forall n, g_{n|n} = \top$. Moreover, if $n$ is an input node of $C$, then we let $h_n$ be an input node labeled in the same way as $n$.

\underline{Recursive case 1:}  Let $r$ and $n$ be two nodes such that $n$ is a right descendant of $r$. Assume that we have constructed in $C'$ a gate computing $g_{r'|n'}$ for every pair $n'$, $r'$ of nodes at distance strictly smaller than the distance between $n$ and $r$ and assume we have constructed $h_s$ for every descendant $s$ of $r$.

Let $(p_1,s_1),\dots,(p_k,s_k)$ be the set of nodes such that $(n,s_i)$ is an input of $p_i$. Observe that $s_i$ and $s_j$ are distinct nodes for $i \neq j$ because the pair $(n,s_i)$ can be the input of at most one $p_i$. Moreover, they have disjoint models. Furthermore, by \cref{lem:decomp}, we have: 
\[ r|n = \bigvee_{i=1}^k (p_i|n) \wedge ((r|n) | p_i) \]
which simplifies to
\[ r|n = \bigvee_{i=1}^k s_i \wedge (r|p_i). \]

Let $s_{k+1}, \dots, s_c$ be the nodes respecting the same vtree node as $s_1,\dots,s_k$ but which are not siblings of $n$. We compute $g_{r|n}$ as
\[\bigvee_{i=1}^k h_{s_i} \wedge g_{r|p_i}  \vee \bigvee_{i=k+1}^c h_{s_i} \wedge \bot. \]
which have all been precomputed by induction.  Moreover, the models of $s_i$ and $s_j$ are all disjoint. Since we also assumed $C$ to be full, we have $\bigvee_{i=1}^c s_i= \top$. Hence we compute $g_{r|n}$ using a partition node. 

\underline{Recursive case 2:} Let $r$ be a node such that $g_{r'|n}$ and $h_{r'}$ have been computed for every descendant $r'$ of $r$. Moreover, assume we have computed $g_{r|n}$ for every right descendant of $r$. Consider the rightmost descendants of $r$. They must be inputs of the circuit, and we assumed that all inputs are labeled by a literal. In particular, either $r$ has two descendants $n_1,n_0$ labeled respectively $x$ and $\neg x$, in which case, we let $r = (x \wedge r|x) \vee (\neg x \wedge r|\neg x) = (x \wedge g_{r|n_1}) \vee (\neg x \wedge g_{r|n_0})$. In this case, we have already constructed $g_{r|n_0}$ and $g_{r|n_1}$ and we have, indeed, a partition node computing $h_r$. 

Or $r$ only has one rightmost descendant $n$ labeled by literal $\ell$. In this case, we let $r = (\ell \wedge g_{r|n}) \vee (\neg \ell \wedge \bot)$ which is also a partition node computing $h_r$.

This concludes the induction: $C'$ is by construction DNNF with only partition nodes. Moreover, the size of $C'$ is quadratic in the size of $C$ by construction. Finally $C'$ contains a gate $h_{\out}$ computing the same function as $C$.

To show that $C'$ is an SDD, it remains to show that $C'$ respects a vtree $T'$. Given a vtree $T$, we define $\tilde{T}$ to be the vtree defined as follows:
\begin{itemize}
\item if $T$ is a leaf labeled by $x$, then $\tilde{T} = T$.
\item Otherwise, let $t_1, \dots, t_n$ be the rightmost path in $T$ and let $T_i$ be the left subtree of $t_i$ for $i<n$ and let $x$ be the variable labeling $t_n$. We define $\tilde T$ as the tree obtained by branching a leaf labeled by $x$ as the left child of the root. The right child of the root is the right path $t_{n-1}', \dots, t_{2}'$ such that the left subtree of $t_i'$ is $\tilde{T_i}$. The right subtree of $t'_{2}$ is $\tilde{T}_{1}$.
\end{itemize}

We can see that $C'$ respects $\tilde{T}$. 
Indeed, let $r$ be a $t$-node of $C$. It can be checked that in $\tilde{T}$, there is a node $\tilde{t}$ such that the left child of $\tilde{t}$ is a leaf labeled by variable $x$,  which is the label of the rightmost leaf below $t$ and the variables below the right subtree of $\tilde{t}$ contain $X_t \setminus \{x\}$. Hence $h_r$ respects $\tilde{t}$.

Now assume $n$ is a right descendant of $r$. Let $u$ be the node respected by $n$, $p$ the parent of $u$ and $s$ the sibling of $u$. By construction of $C'$, node $g_{r|n}$ of $C'$ splits the variables into $\var{s}$ and $\var{t} \setminus (\var{s} \cup \var{u}) = \var{t} \setminus \var{p}$. Now, in $\tilde{T}$, the path from $p$ to $t$ is reversed so there is a node $\tilde{p}$ such that the left subtree of $\tilde{p}$ is $\tilde{T_s}$ hence $\var{\tilde{p}} = \var{s}$ and the right subtree of $\tilde{p}$ contains $\var{t} \setminus \var{p}$. Hence $g_{r|n}$ respects $\tilde{p}$ in $\tilde{T}$ which proves that $C'$ respects $\tilde{T}$, that is, $C'$ is an SDD.

\end{proof}
\end{toappendix}

The resulting SDD, whose construction is given in the appendix, does not respect the same vtree as the original TDD. This is unavoidable. Indeed, consider the \emph{Multiplexer} function $\MUX_n(x_0,\dots,x_{k-1},y_0,\dots,y_{n-1})$ which has $k+n$ variables where $n=2^k$ and is satisfied if and only if $y_{[x]_2} = 1$ where $[x]_2 = \sum_{i=0}^{k-1} x_i \cdot 2^i$. Let $\pi$ be the order $(x_0,\dots,x_{k-1},y_0,\dots,y_{n-1})$. It is not hard to see that there is an OBDD respecting $\pi$ of size $O(n)$ computing $\MUX_n$~\cite[Theorem 4.3.2]{Wegener00}. Hence, there is an SDD of size $O(n)$ respecting $T_\pi$ computing $\MUX_n$. However, one can also prove that any OBDD respecting $\pi^{-1}$ and computing $\MUX_n$ must have size $2^{n}$. Indeed, we have one $Y$-subfunction per assignment of the $Y$ variables: if $\tau,\sigma \in 2^Y$ are distinct, then let $y_i$ be such that, wlog, $1 = \tau(y_i) \neq \sigma(y_i) = 0$. Then $\MUX_n[\tau] \neq \MUX_n[\sigma]$ because they differ on $\alpha_i$, the assignment of $X$ variables encoding $i$ in binary.
In other words, every SDD respecting $T_{\pi^{-1}}$ and computing $\MUX_n$ must have size $2^{n}$. But there is a TDD of size $O(n)$ computing $\MUX_n$ and respecting $T_{\pi^{-1}}$ by \cref{thm:obdd-tdd-linear}. 

We note however that the construction from \cref{thm:tdd<sdd} does not give a compressed, hence not canonical, SDD. It is not clear whether we can always build a polynomial-size canonical SDD equivalent to a given TDD. We leave open the question of comparing canonical SDD and TDD.

\section{Conclusion}

In this paper, we have introduced a new data structure for representing Boolean functions that offers advantages similar to OBDD but can handle bounded treewidth instances. The main advantage of these data structures over the existing ones such as deterministic DNNF or SDD is that they can be minimized into a canonical circuit for which the size and width can be easily understood. While SDDs also have canonical representations, those are not minimal representations, and the canonical representation may be exponentially larger than the minimal one~\cite{van2015role}. Our approach allows to recover compilation results in a clean and modular way.

Several research directions remain concerning TDD. First, it would be interesting to implement a bottom-up compiler with TDD as a target language and perform a comparison with OBDD and SDD compilers. To make TDD competitive, it might be necessary to study a variant that is non-smooth, i.e., allowing $t$-nodes to have as inputs $u$-nodes where $u$ is a descendant of $t$ in the vtree but not necessarily a child. While this can only lead to polynomial size gains, it could make a big difference in practice. Second, and related to this, the question of finding a good vtree in practice remains open. The SDD compiler uses local changes in the vtree to find a better one, and adapting it to TDD may be promising. Vtree changes for the related model of probabilistic circuits have also been studied~\cite{ZhangWAB25}. Generally, we think it would be useful to understand the complexity of transforming a TDD respecting vtree $T$ into an equivalent canonical TDD respecting vtree $T'$, measured in the input and output size.
Finally, an interesting application of TDD that we feel is worth exploring is extending bottom-up compilation to the setting of~\cite{ColnetSZ24} where a conjunction of constraints represented as OBDDs is compiled into a d-DNNF. When the incidence graph of this conjunction has bounded incidence treewidth, and the constraints can be represented by OBDDs with any order, then it is possible to construct an FPT size d-DNNF. It seems that if the constraints are all represented by TDDs using the same vtree, a bottom-up compilation could give similar and slightly more general results. 


\bibliography{main}

\end{document}